\documentclass{article}
\usepackage[utf8]{inputenc}

\usepackage{hhline}
\usepackage{amsmath} 
\usepackage{amsthm}
\usepackage{bbm}
\usepackage{xspace}
\usepackage{caption}
\usepackage{amssymb}
\usepackage{mathtools}
\usepackage{subcaption}
\usepackage{enumitem}
\usepackage{algorithm}
\usepackage{algpseudocode}
\usepackage[utf8]{inputenc}
\usepackage{xcolor}
\usepackage{graphicx}
\usepackage[square,numbers,sort&compress]{natbib}
\usepackage{nicefrac}
\usepackage{array}
\usepackage[toc,page]{appendix}
\usepackage{mathabx}
\usepackage{titlesec}

\newtheorem{theorem}{Theorem}
\newtheorem{proposition}{Proposition}
\newtheorem{corollary}{Corollary}
\newtheorem{lemma}[theorem]{Lemma}
\theoremstyle{definition}
\newtheorem{definition}{Definition}
\theoremstyle{remark}

\allowdisplaybreaks

\newcommand{\naturals}{\mathbb{N}}

\newcommand{\dd}{\mathrm{d}}

\newcommand{\Pcal}{\mathcal{P}}
\newcommand{\Dcal}{\mathcal{D}}
\newcommand{\Acal}{\mathcal{A}}
\newcommand{\Tcal}{\mathcal{T}}
\newcommand{\Lcal}{\mathcal{L}}
\newcommand{\Scal}{\mathcal{S}}
\newcommand{\Zcal}{\mathcal{Z}}
\newcommand{\Xcal}{\mathcal{X}}
\newcommand{\Ycal}{\mathcal{Y}}

\newcommand{\Bcal}{\mathcal{B}}
\newcommand{\Ccal}{\mathcal{C}}

\newcommand{\abs}[1]{\left\lvert#1\right\rvert}

\DeclareMathOperator*{\argmin}{arg\,min}

\newcommand{\Real}{\mathbb{R}}

\newcommand{\roundbrack}[1]{\left( #1 \right)}
\newcommand{\curlybrack}[1]{\left\lbrace #1 \right\rbrace}
\newcommand{\squarebrack}[1]{\left\lbrack #1 \right\rbrack}

\newcommand{\E}{\mathbb{E}}

\newcommand{\Prob}{P}
\newcommand{\Exp}[2]{\mathbb{E}_{#1}\left\lbrack#2\right\rbrack}

\newcommand{\polylog}[1]{\text{polylog}\roundbrack{#1}}

\newcommand{\interval}{\mathcal{I}}
\newcommand{\ghost}{m}
\newcommand{\npo}{{n+1}}

\newcommand{\omaci}{$(1 - \alpha)$-CI~}
\newcommand{\omaps}{$(1 - \alpha)$-PS~}
\newcommand{\labelset}{\Lcal}
\newcommand{\bw}{{(w)}}

\newcommand{\nbin}{N}
\newcommand{\cg}[1]{#1}
\newcommand{\ap}[1]{#1}

\usepackage[margin=1.2in]{geometry}
\usepackage{hyperref}
\usepackage{tikz}
\usetikzlibrary{arrows,positioning} 
\usetikzlibrary{decorations.markings}
\usepackage{subcaption}
\tikzset{
	>=stealth',
	punkt/.style={
		rectangle,
		rounded corners,
		draw=black, very thick,
		text width=8em,
		minimum height=2em,
		text centered},
	pil/.style={
		-latex,
		line width=1pt,
		shorten <=2pt,
		shorten >=2pt
	},
	degil/.style={
		decoration={markings,
			mark= at position 0.5 with {
				\node[transform shape] (tempnode) {$\backslash$};
			}
		},
		postaction={decorate}
	}
}
\usepackage{float}

\ifx false
\title{
	Distribution-free binary classification:\\
	prediction sets, confidence intervals and calibration
}
\author{
	Chirag Gupta\thanks{equal contribution} $^{\1}$, Aleksandr Podkopaev\footnotemark[1] $^{\1,2}$, Aaditya Ramdas$^{1,2}$  \\ 
	Machine Learning Department$^1$\\
	Department of Statistics and Data Science$^2$\\
	Carnegie Mellon University\\
	\texttt{\{chiragg,podkopaev,aramdas\}@cmu.edu}}
\fi

\title{
	Distribution-free binary classification:\\
	prediction sets, confidence intervals and calibration \\
}
\author{
	Chirag Gupta\thanks{equal contribution; paper appeared as a spotlight at NeurIPS 2020.} $^{1}$, Aleksandr Podkopaev\footnotemark[1] $^{1,2}$, Aaditya Ramdas$^{1,2}$ \vspace{0.2in} \\ 
	\texttt{\{chiragg,podkopaev,aramdas\}@cmu.edu}\vspace{0.2in}\\
	Machine Learning Department$^1$\\
	Department of Statistics and Data Science$^2$\\
	Carnegie Mellon University}
\date{\today}

\begin{document}
	\maketitle 
	\begin{abstract}
		We study three notions of uncertainty quantification---calibration, confidence intervals and prediction sets---for binary classification in the distribution-free setting, that is without making any distributional assumptions on the data.
		With a focus towards calibration, we establish a `tripod' of theorems that connect these three notions for score-based classifiers. A direct implication is that distribution-free calibration is only possible, even asymptotically, using a scoring function whose level sets partition the feature space into at most countably many sets. Parametric calibration schemes such as variants of Platt scaling do not satisfy this requirement, while nonparametric schemes based on binning do. To close the loop, we derive distribution-free confidence intervals for binned probabilities for both fixed-width and uniform-mass binning. As a consequence of our `tripod' theorems, these confidence intervals for binned probabilities lead to distribution-free calibration. We also derive extensions to settings with streaming data and covariate shift.
	\end{abstract}
 	\newpage
 	\tableofcontents
	\newpage
 	
	\section{Introduction}\label{sec:intro}
	Let $\Xcal$ and $\smash{\Ycal=\{0,1\}}$ denote the feature and label spaces for binary classification. Consider a predictor $\smash{f : \Xcal \to \Zcal}$ that produces a prediction in some space $\Zcal$. If $\Zcal = \{0, 1\}$, $f$ corresponds to a point prediction for the class label, but often  class predictions are based on a `scoring function'. Examples are, $\Zcal = \Real$ for SVMs, and $\Zcal = [0, 1]$ for logistic regression, random forests with class probabilities, or deep models with a softmax top layer. In such cases, a higher value of $f(X)$ is often interpreted as higher belief that $Y = 1$. In particular, if $\Zcal = [0, 1]$, it is tempting to interpret $f(X)$ as a probability, and hope that
	\begin{equation}
	f(X) \approx \Prob(Y = 1 \mid X) = \Exp{}{Y \mid X}. \label{eq:finest-calibration}
	\end{equation} 
	However, such hope is unfounded, and in general \eqref{eq:finest-calibration} will be far from true without strong distributional assumptions, which may not hold in practice. Valid uncertainty estimates that are related to \eqref{eq:finest-calibration} can be provided, but ML models do not satisfy these out of the box. This paper discusses three notions of uncertainty quantification: calibration, prediction sets (PS) and confidence intervals (CI), defined next. A function $f: \Xcal \to [0,1]$ is said to be (perfectly) calibrated if
	\begin{equation}
	\Exp{}{Y \mid f(X) = a} = a\ \quad \text{ a.s. for all $a$ in the range of $f$. }
	\label{eq:calib_small}
	\end{equation}
	Define the set of all subsets of $\Ycal$, $\smash{\labelset\equiv\{\{0\},\{1\},\{0,1\},\emptyset\}}$, and fix $\alpha\in(0,1)$. A function $\smash{S:  \Xcal \to \labelset}$ is a $(1-\alpha)$-PS  if 
	\begin{equation}\label{eq:PS}
	\Prob( Y \in S(X)) \geq 1 - \alpha.
	\end{equation}In practice, PSs are typically studied for larger output sets, such as $\Ycal_{\text{regression}} = \mathbb{R}$ or $\Ycal_{\text{multiclass}} = \{1, 2, \ldots, L > 2\}$, but in this paper, we pursue fundamental results for binary classification. Finally, let $\interval$ denote the set of all subintervals of $[0,1]$. A function $\smash{C: \Xcal \to \interval}$ is a $(1-\alpha)$-CI if 
	\begin{equation}\label{eq:CI}
	\Prob( \Exp{}{Y \mid X} \in C(X)) \geq 1 - \alpha.
	\end{equation}
	All three notions are `natural' in their own sense, but also different at first sight. We show that they are in fact tightly connected (see Figure~\ref{fig:implication-triangle}), and focus on the implications of this result for calibration.
	Most of our results are in the distribution-free setting, where we are concerned with understanding what uncertainty quantification is possible without making distributional assumptions on the data. This paper is based on the statistical setup of post-hoc uncertainty quantification, described next. 
	
	\textbf{Post-hoc uncertainty quantification setup.} Let $P$ denote the data-generating distribution over $\Xcal \times \Ycal$, and let $(X, Y) \sim P$ denote a general data point. 
	Post-hoc uncertainty quantification is a common paradigm where the available labeled data is split into a \emph{training set} and a \emph{calibration set}. The training set is used to learn a predictor $f : \Xcal \to [0,1]$, and the calibration set is used to supplement $f$ with uncertainty estimates (CIs or PSs), or learn a new calibrated predictor on top of $f$. (In practice, the validation set is often used as the calibration set.) All results in this paper are conditional on the training set; thus the randomness is always over the calibration and test data. We denote the calibration set as $\Dcal_n = \{(X_i, Y_i)\}_{i \in [n]}$, where $n$ is the number of calibration points, and we use the shorthand $[n] := \{1,2, \ldots n\}$. A prototypical test point is denoted as $(X_\npo, Y_\npo)$. The calibration and test data is assumed to be drawn i.i.d. from $P$, denoted succinctly as $\{(X_i, Y_i)\}_{i\in [n+1]} \sim P^{n+1}$. The learner observes realized values of all random variables $(X_i, Y_i)$, except $Y_\npo$. All sets and functions are implicitly assumed to be measurable.
	
	Our work relies on some key ideas in the works of \citet[Section 5]{vovk2005algorithmic}, \citet{foygel2020distribution}, and \citet{zadrozny2001obtaining}. 
	Other related work is cited as needed, and further discussed in Section~\ref{sec:related_work}. All proofs appear ordered in the Appendix.

	\section{Calibration, confidence intervals and prediction sets}\label{sec:calib_ci_ps}
  	
  	A few additional concepts and definitions are needed in order to formally study calibration, CIs and PSs in the distribution-free post-hoc uncertainty quantification setup. These are defined next. 
  	
  	\subsection{Approximate and asymptotic calibration}
  	\label{subsec:approx-cal}
	Calibration captures the intuition of \eqref{eq:finest-calibration} but is a weaker requirement, and was first studied in the meteorological literature for assessing probabilistic rain forecasts \citep{brier1950verification, sanders1963subjective, murphy1967verification, dawid1982well}. \citet{murphy1967verification} described the ideal notion of calibration, called \emph{perfect calibration} \eqref{eq:calib_small},
	which has also been referred to as \emph{calibration in the small} \citep{vovk2012venn}, or sometimes simply as \emph{calibration} \citep{guo2017nn_calibration, vaicenavicius2019bayescalibration, dawid1982well}. The types of functions that can achieve perfect calibration can be succinctly captured as follows.
	\begin{proposition}
		\label{prop:calibration-characterization}
		A function $f : \Xcal \to [0,1]$  is perfectly calibrated if and only if there exists a  space $\Zcal$ and a  function $g: \Xcal \to \Zcal$, such that 
		\begin{equation}
		f(x) =  \Exp{}{Y \mid g(X)=g(x)} \quad \text{almost surely } P_X.
		\label{eq:calibration-characterization}
		\end{equation}
	\end{proposition}
	In other words, $f$ is calibrated if and only if there exists another function $g$ such that $f$ is the expected value of $Y$ given the output of $g$. Vaicenavicius et al. \cite{vaicenavicius2019bayescalibration} stated and gave a short proof for the `only if' direction. While the other direction is also straightforward, together they lead to an appealingly simple and complete characterization.
	The proof of Proposition~\ref{prop:calibration-characterization} is in Appendix~\ref{appsec:calib_ci_ps}.
	
	It is helpful to consider two extreme cases of Proposition~\ref{prop:calibration-characterization}. First, setting $g$ to be the identity function yields that the Bayes classifier $\Exp{}{Y|X}$ is perfectly calibrated. Second, setting $g(\cdot)$ to any constant implies that $\Exp{}{Y}$ is also a perfect calibrator. 
	Naturally, we cannot hope to estimate the Bayes classifier without assumptions, but even the simplest calibrator $\Exp{}{Y}$ can only be approximated in finite samples. Since Proposition~\ref{prop:calibration-characterization} states that calibration is possible iff the RHS of \eqref{eq:calibration-characterization} is known exactly for some $g$,  perfect calibration is impossible in practice.
	Thus we resort to satisfying the requirement~\eqref{eq:calib_small} approximately, which is implicitly the goal of many empirical calibration techniques. 
	\begin{definition}[Approximate calibration]\label{def:app_calib}
		A predictor $f:\Xcal\to [0,1]$ is $(\varepsilon, \alpha)$-calibrated for some $\varepsilon, \alpha\in [0,1]$ 
		if with probability at least $1-\alpha$, 
		\begin{equation}
		\abs{\Exp{}{Y | f(X)}-f(X)}\leq \varepsilon.
		\label{eq:approximate-calibration}
		\end{equation}
	\end{definition}
	
	Clearly, every predictor $f$ is $(1,0)$-calibrated and $(0,1)$-calibrated. Further, if $f$ is $(\varepsilon,\alpha)$-calibrated, then it is also $(\varepsilon',\alpha)$-calibrated for $\varepsilon' > \varepsilon$ and $(\varepsilon,\alpha')$-calibrated for $\alpha' > \alpha$, and so we are typically only interested in the smallest ``pareto optimal boundary'' pairs of $(\varepsilon,\alpha)$ for which approximate calibration holds, or specifically for a fixed $\alpha$ like 0.1, what is the smallest $\varepsilon$ for which calibration holds.
	
		

	Suppose $f$ is not approximately calibrated for  small values of $\varepsilon$ and $\alpha$. As mentioned in the Introduction, we can `recalibrate' $f$ using a post-hoc calibration algorithm $\Acal$. Such an $\Acal$ takes $f$ (learnt on the training data) as input along with independent calibration data $\Dcal_n = \{(X_i, Y_i)\}_{i \in [n]}$, and outputs $ \Acal(\Dcal_n, f)=h_n : \Xcal \to [0,1]$, a predictor with presumably improved calibration properties compared to the original $f$. This setup was popularized by \citet{guo2017nn_calibration}; in their work, $f$ is a deep neural network and a proposed algorithm $\Acal$ is temperature scaling. In this paper, we study when $\Acal$ can be shown to satisfy \emph{distribution-free approximate calibration}:
	\begin{equation}
	    P^{n+1}(\abs{\Exp{}{Y | h_n(X_{n+1})}-h_n(X_{n+1})}\leq \varepsilon) \geq 1- \alpha \quad \text{for every $f, P$.}\label{eq:df-calibration}
	\end{equation}
	Above, $P^{n+1}$ denotes the product distribution of the i.i.d. calibration and test points: $\{(X_i, Y_i)\}_{i \in [\npo]} \sim P^{n+1}$. Note that $h_n = \Acal(\Dcal_n, f)$ is random over the calibration data $\Dcal_n$; we reinforce this by writing an $n$ in the subscript. 
	In the limit of infinite calibration data, a good calibration algorithm should guarantee approximate calibration with vanishing $\varepsilon$. This is formalized in the upcoming definition of asymptotic calibration. 
    We use $(\Xcal \times \Ycal)^* = \bigcup_{n\in \naturals} (\Xcal \times \Ycal)^n$ to denote the space of the calibration data for arbitrary $n$, and $[0,1]^{\Xcal}$ to denote a function from $\Xcal$ to $[0,1]$ (such as $f$). 
    
	
	
	\begin{definition}[Distribution-free asymptotic calibration] \label{def:asymp_calib}
    	A post-hoc calibration algorithm $\Acal : (\Xcal \times \Ycal)^* \times [0,1]^\Xcal \to [0,1]^\Xcal$ is said to be distribution-free asymptotically calibrated if there exists an $\alpha \in (0, 0.5)$ and a $[0,1]$-valued sequence  $\{\varepsilon_n\}_{n \in \mathbb{N}}$ with $\lim_{n \to \infty} \varepsilon_n  = 0$, such that for every $n$, $h_n = \Acal(\Dcal_n, f)$  satisfies condition \eqref{eq:df-calibration} with parameters $(\varepsilon_n, \alpha)$. 
	\end{definition}
	\noindent Note that condition \eqref{eq:df-calibration} requires approximate calibration not only over all $P$, but also over all $f$. Thus asymptotic calibration requires $\Acal$ to calibrate \emph{any fixed} $f$ over \emph{all distributions} $P$.
	
	
    \subsection{Prediction sets and confidence intervals with respect to $f$}
	\label{subsec:ps-ci-f}
	
	
	To motivate a new definition of PSs and CIs with respect to $f$, we review a recent result on distribution-free CIs by \citet{foygel2020distribution}, where the existence of `informative' distribution-free CIs was discussed. 
	
	PSs and CIs are only `informative' if the sets or intervals produced by them are small. To quantify this, we measure CIs using their width (denoted as $\abs{C(\cdot)}$), and PSs using their diameter (defined as the width of the convex hull of the PS). For example, in the case of binary classification, the diameter of a PS is $1$ if the prediction set is $\{0,1\}$, and $0$ otherwise (since $Y \in \{0,1\}$ always holds, the set $\{0,1\}$ is `uninformative'). A short CI such as $[0.39,0.41]$ is more informative than a wider one such as $[0.3, 0.5]$.  
	
	For a given distribution, one might expect the diameter of a $(1-\alpha)$-PS to be larger than the width of a $(1-\alpha)$-CI, since we want to cover the actual value of $Y$ and not its conditional expectation. As an example,  if $\Exp{}{Y|X=x} = 0.5$ for every $x$, then the shortest possible CI is $(0.5,0.5]$ whose diameter is $0$. However, a $(1-\alpha)$-PS has no choice but to output $\{0,1\}$ for at least $(1-2\alpha)$ fraction of the points (and a random guess for the other $2\alpha$ fraction), and thus must have expected diameter $\geq 1- 2\alpha$ even in the limit of infinite data. 
	
	Recently, \citet{foygel2020distribution} built on an earlier result of \citet{vovk2005algorithmic} to show that if an algorithm provides  \omaci for all product distributions $P^{n+1}$ (of the training data and test point), then it also provides a \omaps whenever the distribution of $P_X$ is nonatomic, that is, it does not contain any atoms or `point masses'. (If the CI function is $C : \Xcal \to \interval$, then the corresponding PS function would be $\smash{S(\cdot) = C(\cdot) \cap \{0, 1\}}$.) 
	Since this implication holds for all nonatomic distributions $P_X$, including the ones with $\Exp{}{Y|X} \equiv 0.5$ discussed above, it implies that distribution-free CIs must necessarily be wide. Specifically, their widths cannot shrink to $0$ as $n \to \infty$. This can be treated as an impossibility result for the existence of informative distribution-free CIs.
	
	One way to circumvent the above impossibility result is to consider CIs at a `coarser resolution'. We introduce the notion of a CI or PS `with respect to a function $f$' (w.r.t. $f$). 

		\begin{definition}[CI or PS w.r.t. $f$]\label{def:f-confidence-interval}
		Fix a predictor $f : \Xcal \to [0,1]$ and let $(X,Y)\sim P$. A function $\smash{C :[0,1] \to \interval}$ is said to be a $\smash{(1-\alpha)}$-CI  with respect to $f$ if 
		\begin{equation}
		\Prob( \Exp{}{Y \mid f(X)} \in C(f(X)) ) \geq 1 - \alpha.  \label{eq:f-confidence-interval}
		\end{equation}
		Analogously, a function $\smash{S :[0,1]\to\labelset}$ is a $(1-\alpha)$-PS  with respect to $f$ if
		\begin{equation}
		\Prob( Y \in S(f(X)) ) \geq 1 - \alpha.  \label{eq:f-prediction-set}
		\end{equation}
	\end{definition}
	
	These definitions can be extended in a natural way if $\text{Range}(f) \neq [0,1]$, as we do in the conference version of this paper \citep{gupta2020distribution}. If $f$ is injective (one-to-one), then \eqref{eq:f-confidence-interval} and \eqref{eq:f-prediction-set} reduce to \eqref{eq:CI} and \eqref{eq:PS}. The more interesting (and typical) case is when $f$ is not injective. In this case, the level sets of $f$ partition $\Xcal$ at a coarser `resolution': $\Xcal= \cup_{z \in [0,1]}\{x : f(x) = z\}$, and we can ask the (easier) question of producing a single CI or PS with respect to every $z \in [0,1]$, instead of every $x \in \Xcal$. 
	
	Naturally, for \eqref{eq:f-confidence-interval} or \eqref{eq:f-prediction-set} to hold, the functions $C$ and $S$ must depend on $P$. Similar to the post-hoc calibration setting, we ask if $C$ or $S$ can be \emph{learnt} using independent calibration data $\Dcal_n$ drawn from $P$. 
    Let $\Ccal$ denote an algorithm that produces a CI function using $f$ and $\Dcal_n$, $C_n=\Ccal(\Dcal_n, f) : [0,1] \to \interval$, where the notation $C_n$ reinforces the dependence of the CI function on $\Dcal_n$. Similarly, let $\Scal$ denote an algorithm that produces a PS function, $S_n =\Scal(\Dcal_n, f) : [0,1] \to \labelset$. 
	Akin to distribution-free approximate calibration \eqref{eq:df-calibration}, we have the following definitions for distribution-free CIs and PSs. $C_n$ is said to be a distribution-free CI w.r.t.\ a fixed $f$ if 
	\begin{equation}
	    P^{n+1}( \Exp{}{Y_{n+1} \mid f(X_{n+1})} \in C_n(f(X_{n+1})) ) \geq 1 - \alpha \quad  \text{for every $P$,} \label{eq:df-ci}
	\end{equation}
	and $S_n$ is said to be a distribution-free PS w.r.t. a fixed $f$ if 
	\begin{equation}
	    P^{n+1}( Y_{n+1} \in S_n(f(X_{n+1})) ) \geq 1 - \alpha \quad \text{for every $P$.}\label{eq:df-ps}
	\end{equation}
	
	Table~\ref{tab:my_label} summarizes the notation introduced so far. In the rest of the paper, whenever we refer to objects with an `$n$' in the subscript such as $h_n, C_n, S_n$, they should be understood as the outputs of some algorithms $\Acal, \Ccal, \Scal$ when supplied with input $\Dcal_n$ and $f$.
	
	\cg{
	\subsection{When is distribution-free post-hoc uncertainty quantification possible?}
	
	Are distribution-free guarantees such as \eqref{eq:df-calibration}, \eqref{eq:df-ci}, and \eqref{eq:df-ps} too restrictive, or can they be achieved? We show that the answer for calibration and CIs (roughly) depends on how `large' the range of $f$ is. The result of \citet{foygel2020distribution} implies that if $f$ is injective---that is $f$ maps unique elements to unique elements---then informative distribution-free CIs are impossible. On the other hand, if $f$ maps all of $\Xcal$ to a single element,  a short interval around the empirical mean of the $Y_i$'s achieves 
	\eqref{eq:df-ci} since $\Exp{}{Y \mid f(X)} = \Exp{}{Y}$. In this work, we characterize the transition point between these two behaviors.
	
	In Section~\ref{sec:equivalence}, we extend the above impossibility result to all functions $f$ whose range contains any sub-interval of $[0,1]$, a condition satisfied by all parametric machine learning models. On the other hand, in Section~\ref{sec:dfc_guarantees} we propose algorithms that achieve distribution-free CIs for $f$ with finite range. We also show a close relationship between approximate calibration and CIs w.r.t.\ $f$. Based on this relationship, the results for distribution-free CIs extend to distribution-free calibration, and vice-versa. Specifically, no parametric (post-hoc) calibration algorithm, such as Platt scaling \citep{Platt99probabilisticoutputs} or temperature scaling \citep{guo2017nn_calibration}, can be distribution-free calibrated. On the other hand, distribution-free calibration guarantees can be shown for the discrete binning method of histogram binning \citep{zadrozny2001obtaining}.
	
	In contrast to CIs and calibration, it is well known that meaningful and informative distribution-free PSs can be produced for any $f$, using a technique known as split conformal prediction \citep{papadopoulos2002inductive}. The broader literature on (non-split) conformal prediction also deals with techniques that produce distribution-free PSs without fixing an $f$ learnt on a separate split of the data \citep{vovk2005algorithmic, gupta2019nested}. We do not discuss algorithmic results for distribution-free PSs in this paper and refer the reader to one of the aforementioned papers on conformal prediction.}

		\begin{table}[]
	
	    \centering
	    \begin{tabular}{|c|c|}
           \hline 
           Calibration data & $\Dcal_n = \{(X_i, Y_i)\}_{i \in [n]}$ \\
          \hline 
           Test point & $(X_{n+1}, Y_{n+1})$ \\
         \hline 
           General data point & $(X, Y)$ \\
            \hhline{|=|=|}
	       Probability over i.i.d. calibration and test data & $P, P^{n+1}$ \\ 
	       \hhline{|=|=|}
	       Predictor learnt on (a split of the) training data  & $f : \Xcal \to [0,1]$ \\ 
	       \hline
	       General functions with unspecified sources of randomness  & $f,  C, S$ \\ 
	       \hline
	       Random functions of the calibration data $\mathcal{D}_n$ & $h_n, C_n, S_n$\\
    	   \hline
	    \end{tabular}
	    \caption{Notation used in this paper to study post-hoc uncertainty quantification. }
	    \label{tab:my_label}
	\end{table}
	
	\section{Relating the notions of uncertainty quantification} 
	\label{sec:equivalence}
	The relationships between the notions of uncertainty quantification are summarized in Figure~\ref{fig:implication-triangle}. In this figure, and in the rest of the section, we denote the distribution of the random variable $Z = f(X)$ as $P_{f(X)}$. 
	In Section~\ref{subsec:calib_and_ci}, we show that if an algorithm provides a CI w.r.t. $f$, it can be used to provide approximate calibration and vice-versa (Theorem~\ref{thm:calib_ci_equiv}). 
	In Section~\ref{subsec:PS-CI-equivalence}, we show that if an algorithm constructs a distribution-free CI w.r.t. $f$, then the constructed CIs must also be PSs for a large class of distributions $P$ for which $P_{f(X)}$ is nonatomic (Theorem~\ref{thm:PS-CI-equivalence}). Since we expect the width of CIs to be shorter than the diameter of PSs, this can be interpreted as an impossibility result for informative distribution-free CIs (Corollary~\ref{cor:CI-PS-diameter}). 
	Merging these two results, in Section~\ref{subsec:partition}, we show 
	that meaningful distribution-free calibration is not possible for certain scoring functions and post-hoc calibration algorithms (Theorem~\ref{thm:f-must-be-coarse}). 
	
	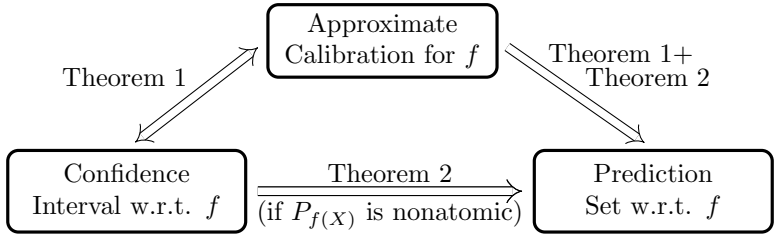
\begin{figure}[t]
		\centering
		\begin{tikzpicture}[node distance=1cm, auto]
		\node[punkt] (calib) at (-0.1,0) {Approximate Calibration for $f$};
		\node[punkt, minimum width=3cm, inner sep=5pt] (cs) at (-3.5,-2) {Confidence Interval w.r.t. $f$};
		\node[punkt, inner sep=5pt] (ps) at (3.5,-2) {Prediction Set w.r.t. $f$};
		\draw[-implies,double distance = 0.1cm, shorten >=4pt, shorten <=4pt] (cs.east) to node [above] {Theorem~\ref{thm:PS-CI-equivalence}} (ps.west);
		\draw[-implies,double distance = 0.1cm, shorten >=4pt, shorten <=4pt] (cs.east) to node [below] {(if $P_{f(X)}$ is nonatomic)} (ps.west);
		\draw[-implies,double distance = 0.1cm, shorten >=4pt, shorten <=4pt] (calib.east) to node [above right] {\Large $\substack{\hspace{-0.6cm}\text{Theorem~\ref{thm:calib_ci_equiv}+ }\\ \text{Theorem~\ref{thm:PS-CI-equivalence}}}$} (ps.north);
		\draw[implies-implies,double distance = 0.1cm, shorten >=4pt, shorten <=4pt] (cs.north) to node [above left] {Theorem~\ref{thm:calib_ci_equiv}} (calib.west);
		\end{tikzpicture}
		\caption{Relationship between notions of distribution-free uncertainty quantification. }
		\label{fig:implication-triangle}
	\end{figure}
	
	
	\subsection{Relating calibration and confidence intervals}\label{subsec:calib_and_ci}
	Suppose we are given a predictor $f : \Xcal \to [0,1]$ that is $(\varepsilon,\alpha)$-calibrated. Then one can construct a function $C$ that is a $(1-\alpha)$-CI: for $x \in \Xcal$,
	\begin{equation}\label{eq:calib_to_ci}\small
	\underbrace{\abs{\Exp{}{Y \mid f(x)} - f(x)} \leq 
	\varepsilon}_{\text{calibration}} 
	\implies  \underbrace{\Exp{}{Y \mid f(x)} \in C(f(x))}_{\text{CI w.r.t. $f$}}:=
	[f(x)-\varepsilon,f(x)+\varepsilon].
	\end{equation}
	On the other hand, given $C : [0,1] \to \interval$ that is a $(1-\alpha)$-CI w.r.t. $f$, define for $z\in[0,1]$, the left-endpoint, right-endpoint, and midpoint functions respectively:
	\begin{equation}
	\small
	u_C(z) := \sup\curlybrack{t:t\in C(z)}, \ l_C(z) := \inf\curlybrack{t:t\in C(z)}, \ m_C(z) := 
	(u_C(z)+l_C(z))/2.
	\end{equation}
	Consider the midpoint $m_C(f(x))$ as a `corrected' prediction for $x\in\Xcal$:
	\begin{equation}\label{eq:calib_predictor}
	\widetilde{f}(x) := m_C(f(x)),\ x\in\Xcal,
	\end{equation}
	and let $\smash{\varepsilon=\sup_{z\in \text{Range}(f)}\curlybrack{|C(z)|/2}}$ be  the largest interval radius.
	Then $\widetilde{f}$ is $(\varepsilon,\alpha)$-calibrated. These claims are formalized next. 
	\begin{theorem}\label{thm:calib_ci_equiv}
		Fix any $\alpha\in(0,1)$. Let $\smash{f:\Xcal\to [0,1]}$ be a predictor that is $(\varepsilon,\alpha)$-calibrated for some $\varepsilon \in (0,1)$. Then the function $C$ in \eqref{eq:calib_to_ci} is a $(1-\alpha)$-CI with respect to $f$. 
		
		Conversely, fix a scoring function $\smash{f: \Xcal \to [0,1]}$. If $C$ is a $\smash{(1-\alpha)}$-CI with respect to $f$, then the predictor $\widetilde{f}$ in~\eqref{eq:calib_predictor} is $(\varepsilon,\alpha)$-calibrated for $\smash{\varepsilon=\sup_{z\in [0,1]}\curlybrack{|C(z)|/2}}$.
	\end{theorem}
	
	\ifx false
	The proof is in Appendix~\ref{appsec:equivalence}. An implication of Theorem~\ref{thm:calib_ci_equiv} is that having a sequence of predictors that is asymptotically calibrated yields a sequence of confidence intervals with vanishing width as $\smash{n\rightarrow \infty}$. 
	This is formalized in the following corollary, also proved in Appendix~\ref{appsec:equivalence}. 
	
	\begin{corollary}\label{cor:asymp_calib_CI_diam}
		Fix any $\alpha\in(0,1)$. If a sequence of predictors $\{f_n\}_{n\in\naturals}$ is asymptotically calibrated at level $\alpha$, then construction~\eqref{eq:calib_to_ci} yields a sequence of functions $\{C_n\}_{n \in \mathbb{N}}$ such that each $C_n$ is a $(1-\alpha)$-CI with respect to $f_n$ and 
		$  \lim_{n \to \infty}\sup_{z \in [0,1]}\abs{C_n(z)}=0$.    
	\end{corollary}
	\fi 
	The proof of the theorem is in Appendix~\ref{appsec:equivalence}.  Note that Theorem~\ref{thm:calib_ci_equiv}
	is not restricted to the post-hoc uncertainty quantification setting and the calibration and CI functions need not satisfy distribution-free guarantees as defined in \eqref{eq:df-calibration} or \eqref{eq:df-ci}. In contrast, the relationship between CIs and PSs stated in the following subsection is specific to the distribution-free setting. 
	\subsection{Relating confidence intervals and prediction sets in the distribution-free setting}
	\label{subsec:PS-CI-equivalence}
	\cg{In this section, we relate CIs and PSs with respect to a fixed function $f : \Xcal \to [0,1]$. Consider the following set of distributions, whose motivation becomes clearly shortly: 
	\begin{equation}
	    \Pcal_f := \{\text{distributions $P$ over $\Xcal \times \Ycal$ } : P_{f(X)} \text{ is nonatomic}\}.
	    \label{eq:p-f-family}
	\end{equation}
	$P_{f(X)}$ being nonatomic means that the distribution of $f(X)$, when $(X,Y) \sim P$, contains no atoms or `point masses'. 
	Suppose $C_n$ satisfies \eqref{eq:df-ci}, that is, it provides a CI guarantee w.r.t. $f$ for all distributions $P$. We show that $C_n$ can be used to provide a modified PS guarantee which is not distribution-free but holds for all $P \in \Pcal_f$:
	\begin{equation}
	    P^{n+1}( Y_{n+1} \in S_n(f(X_{n+1}))) \geq 1 - \alpha \quad \text{for every $P \in \Pcal_f$.}
	    \label{eq:df-ps-2}
	\end{equation}
	The following result is proved in Appendix~\ref{appsec:equivalence}. 
	\begin{theorem}
		\label{thm:PS-CI-equivalence}
		Fix $f : \Xcal \to [0, 1]$ and $\smash{\alpha \in (0, 1)}$. If $C_n$ is a  distribution-free confidence interval with respect to $f$, as in \eqref{eq:df-ci}, then $S_n(\cdot) = C_n(\cdot) \cap \{0,1\}$ is a $(1-\alpha)$-prediction set with respect to $f$ for every $P \in \Pcal_f$, as in \eqref{eq:df-ps-2}.
	\end{theorem}}
	
	Above we transformed the CI function $C_n$ to a PS function $S_n$ by performing an intersection with the $\{0,1\}$. 
	Based on the intuition discussed before Definition~\ref{def:f-confidence-interval}, Theorem~\ref{thm:PS-CI-equivalence} can be interpreted as an impossibility result for distribution-free valid CIs that are `informative' for all distributions. 
	\begin{corollary}
		\label{cor:CI-PS-diameter}
		Fix $f : \Xcal \to [0,1]$ and $\alpha \in (0, 0.5)$. If $C_n$ is a distribution-free confidence interval with respect to $f$ \eqref{eq:df-ci}, 
		and $\Pcal_f$ is non-empty, then there exists a distribution $P \in \Pcal_f$ such that 
		\begin{equation*}
		\E_{P^{n+1}}{|C_n(f(X_\npo))|} \geq 0.5 - \alpha. 
		\end{equation*}
	\end{corollary}
	
	Note that for every $P$, there exists a CI function with expected width equal to zero: $C_P(\cdot) = \{\Exp{P}{Y \mid f(X) = \cdot}\}$. A desirable property for $C_n$ is consistency: given enough samples from $P$, does $C_n$ recover $C_P$? Corollary~\ref{cor:CI-PS-diameter} shows that if $\Pcal_f$ is non-empty, then no distribution-free CI function can be `distribution-free consistent' for $C_P$ --- there exist $P \in \Pcal_f$ for which the average width of the CI is lower bound by a constant independent of $n$. 
	
	Thus we would like to know when $\Pcal_f$ is non-empty. First, note that if the range of $f$ is countable, then for any $P$, $P_{f(X)}$ contains atoms (due to the subadditivity of measure, any distribution over a countable set must contain atoms). Thus $\Pcal_f$ is empty and Corollary~\ref{cor:CI-PS-diameter} does not apply. On the other hand, Lemma~\ref{lemma:p_f_interval} in Appendix~\ref{appsec:characterizing-f-pf} shows that if the range of $f$ is $[0,1]$ or contains any sub-interval of $[0,1]$, then $\Pcal_f$ is non-empty (the proof relies on a technical probability theory result of \citet{ershov1975extension}). Thus Corollary~\ref{cor:CI-PS-diameter} applies to all standard parametric machine learning models, whose range is usually $[0,1]$ or $(0,1)$. 
	In the following subsection, we use Corollary~\ref{cor:CI-PS-diameter} to show an impossibility result for certain post-hoc calibration algorithms.

	\subsection{Impossibility result for distribution-free post-hoc calibration}
	\label{subsec:partition}
 	Proposition~\ref{prop:calibration-characterization} shows that a function $f$ is  calibrated if and only if it takes the form~\eqref{eq:calibration-characterization} for some function $g$. 
    Observe that 
    $g$ essentially provides a \emph{partition} of $\Xcal$ based on the level sets of $g$. Denote this partition as $\{\Xcal_z\}_{z \in \Zcal}$, where     
	$ \smash{\Xcal_z = \{x \in \Xcal : g(x) = z\}}$. 
	Then we may equivalently define $f$ in \eqref{eq:calibration-characterization} through a
    set of values $\smash{\{f_z = P(Y = 1 \mid X \in \Xcal_z)\}_{z \in \Zcal}}$, setting $f(\cdot) = f_{g(\cdot)}$. 
	In this sense, calibration can be viewed as a goal with two parts: (A) identify a `meaningful' partition of $\Xcal$ and (B) estimate the conditional probabilities for each partition. 
	
	\begin{corollary}[to Proposition~\ref{prop:calibration-characterization}]
		Any calibrated classifier $f$ is characterized by an index set $\Zcal$, \begin{enumerate}[label=(\Alph*)] \item a partition of $\Xcal$ into subsets $\{\Xcal_z\}_{z \in \Zcal}$, and \item corresponding conditional probabilities $\{f_z\}_{z \in \Zcal}$. \end{enumerate}
		\label{cor:calibration-characterization-partition}
	\end{corollary}

    This interpretation motivates the underlying principle of post-hoc calibration.  
    Existing ML techniques often implicitly do (A). They produce $f$ that, while miscalibrated, may have some rough monotonicity with respect to the true probability: $f(x_1) \geq f(x_2) \iff \Prob(Y = 1 \mid X = x_1) \geq \Prob(Y = 1 \mid X = x_2)$ (see \citet[Figures 1 and 2]{zadrozny2002transforming} for examples when such a hypothesis roughly holds on real data). In other words, the partitioning of $\Xcal$ induced by the level sets of $f$, $\{\Xcal_z = \{x : f(x) = z\}\}_{z \in [0,1]}$, is often informative, but $\abs{z - \Prob(Y = 1 \mid X_\npo \in \Xcal_z)}$ may be large. Post-hoc calibration techniques leverage the solution of (A) provided by $f$, and focus on (B); they use calibration data $\Dcal_n$ to estimate $\Prob(Y = 1 \mid X \in \Xcal_z)$ for every $z \in \text{Range}(f)$.

Thus a post-hoc calibration method `recalibrates' $f$ by mapping its output  to a new value in $[0,1]$. Let $h_n=\Acal(\Dcal_n, f)$ be the output of a post-hoc calibration method $\Acal$ and let $m_n : [0,1] \to [0,1]$ be the implicit mapping function so that $h_n(x) = m_n(f(x))$. Consider three popular parametric algorithms for post-hoc calibration: Platt scaling~\citep{Platt99probabilisticoutputs}, temperature scaling~\citep{guo2017nn_calibration}, and beta calibration~\citep{kull2017beyond}. The mapping $m_n$ learnt by each of these methods is strictly monotonic, and hence, injective (one-to-one).\footnote{This assumes that the parameters satisfy natural constraints as discussed in the original papers: $a, b \geq 0$ for beta scaling with at least one of them nonzero, $ A < 0$ for Platt scaling and $T>0$ for temperature scaling.} Let us call these as `injective' post-hoc calibration algorithms. We now state the impossibility result for distribution-free calibration. 


	\begin{theorem} \label{thm:f-must-be-coarse}
        It is impossible for an injective post-hoc calibration algorithm to be distribution-free asymptotically calibrated.
	\end{theorem}
	\ifx false 
	In words, 
	the cardinality of the partition induced by $f_n$ must be at most countable for large enough $n$. The following phrasing is convenient: $f$ is said to lead to a \emph{fine partition} of $\Xcal$ if $\smash{|\Xcal^{(f)}| > \aleph_0}$. Then, for the purposes of distribution-free asymptotic calibration, Theorem~\ref{thm:f-must-be-coarse} necessitates us to consider $f$ that do not lead to fine partitions. Popular scoring functions such as logistic regression, deep neural-nets with softmax output and SVMs lead to continuous $f$ that induce fine partitions of $\Xcal$ and thus cannot be asymptotically calibrated without distributional assumptions. 
	\fi

	The proof of Theorem~\ref{thm:f-must-be-coarse} is in Appendix~\ref{appsec:equivalence}, but we briefly sketch its intuition below. Since the mapping $m_n$ produced by $\Acal$ is injective, $\Exp{}{Y\mid h_n(X)} = \Exp{}{Y\mid m_n(f(X))}=\Exp{}{Y\mid f(X)}$. Thus a CI w.r.t. $h_n$ is also a CI w.r.t. $f$. As a consequence, if $h_n$ is distribution-free $(\varepsilon_n, \alpha)$-calibrated, then by Theorem~\ref{thm:calib_ci_equiv}, 
	\begin{equation*}
	C_n(f(X)):=
	[h_n(X)-\varepsilon_n, h_n(X)+\varepsilon_n] = [m_n(f(X))-\varepsilon_n, m_n(f(X))+\varepsilon_n],
	\end{equation*} 
	is a distribution-free $(1-\alpha)$-CI w.r.t. $f$. Consider any standard parametric function $f$. As shown in Appendix~\ref{appsec:characterizing-f-pf}, $\Pcal_f$ is non-empty for such $f$. We can thus use Corollary~\ref{cor:CI-PS-diameter} to conclude that the width of any distribution-free CI such as $C_n$ must be lower bounded by $0.5 - \alpha$ (for all $n$). Thus, $2\varepsilon_n \geq 0.5 - \alpha$ for all $n$, which is a constant lower bound on $\varepsilon_n$ (since $\alpha < 0.5)$. We conclude that $\lim_{n \to \infty} \varepsilon_n > 0$, and asymptotic calibration is impossible. 
	
	
	\cg{The implication of Theorem~\ref{thm:f-must-be-coarse} is that injective algorithms such as Platt scaling, temperature scaling, and beta scaling cannot satisfy distribution-free calibration in any meaningful way. While all parameteric post-hoc calibration methods we are aware of are injective, we conjecture that a result like  Theorem~\ref{thm:f-must-be-coarse} holds even more generally for any parametric post-hoc calibration method, as long as its output is continuous.} 
	
	Nonparametric calibration methods of isotonic regression \citep{zadrozny2002transforming} and histogram binning \citep{zadrozny2001obtaining} are not injective, and thus can potentially satisfy distribution-free asymptotic calibration guarantees. 
	In the following section, we analyze histogram binning and show that any scoring function can be `binned' to achieve distribution-free calibration. We explicitly quantify the finite-sample  approximate calibration guarantees that automatically also lead to asymptotic calibration. We also discuss calibration in the online setting and calibration under covariate shift. 
	
	\ifx false 
	The proof of Theorem~\ref{thm:f-must-be-coarse} is in Appendix~\ref{appsec:equivalence}, but we briefly sketch its intuition below. Corollary~\ref{cor:asymp_calib_CI_diam} shows that asymptotic calibration allows construction of CIs whose widths vanish asymptotically. Corollary~\ref{cor:CI-PS-diameter} shows however that asymptotically vanishing CIs are impossible (without distributional assumptions) for $f$ if there exists a distribution $P$ such that $P_{f(X)}$ is nonatomic. Consequently asymptotic calibration is also impossible for such $f$. 
	If $\Zcal = \text{Range}(f)$ is countable, then by the axioms of probability, 
	$\sum_{z \in \Zcal} \Prob(X\in \Xcal_z) = \Prob(X \in \Xcal)  = 1,$
	and so $\Prob(X \in \Xcal_z) \neq 0$ for at least some $z$. Thus $P_{f(X)}$ cannot be nonatomic for any $P$. On the other hand, if $\Zcal$ 
	is uncountable we can show that there always exists a $P$ such that $P_{f(X)}$ is nonatomic. 
	Hence distribution-free asymptotic calibration is impossible for such $f$, which leads to the result of the theorem. 
	
	Theorem~\ref{thm:f-must-be-coarse} also applies 
	to many parametric calibration schemes that `recalibrate' an existing $f$ through a wrapper $h_n : \Zcal \to [0, 1]$ learnt on the calibration data, with the goal 
	that $h_n(f(\cdot))$ is nearly calibrated: $\Exp{}{Y\mid h_n(f(X))  }\approx h_n(f(X))$. 
	For instance, consider methods like Platt scaling~\citep{Platt99probabilisticoutputs}, temperature scaling~\citep{guo2017nn_calibration} and beta calibration~\citep{kull2017beyond}. Each of these methods learns a continuous and monotonic\footnote{This assumes that the parameters satisfy natural constraints as discussed in the original papers: $a, b \geq 0$ for beta scaling with at least one of them nonzero, $ A < 0$ for Platt scaling and $T>0$ for temperature scaling.} (hence bijective) wrapper $h_n$, and thus ${\Exp{}{Y\mid h_n(f(X))}=\Exp{}{Y\mid f(X)}}$. If $h_n$ is a good calibrator, we would have ${\Exp{}{Y\mid f(X)}\approx h_n(f(X))}$. One way to formalize this is to consider whether an interval around $h_n(f(X))$ is a CI for $\Exp{}{Y \mid f(X)}$. In other words --- does there exist a function $\varepsilon_n: \Zcal \to [0, 1]$  such that for every distribution $P$, 
	\begin{equation*}
	\widetilde{C}_n(f(X)):=
	[h_n(f(X))-\varepsilon_n(f(X)), h_n(f(X))+\varepsilon_n(f(X))]
	\end{equation*}
	is a $(1-\alpha)$-CI with respect to $f$, and $\varepsilon_n(f(X)) = o_P(1)$? Theorem~\ref{thm:f-must-be-coarse} (when converted to the CI form, per Section~\ref{subsec:calib_and_ci}) shows that such an asymptotically shrinking CI does not exist if $f$ leads to a fine partition of $\Xcal$. 
	Thus the aforementioned parametric calibration methods cannot lead to asymptotically calibrated functions $h_n$ (without distributional assumptions). We conjecture that this impossibility result holds more generally for any continuous parametric method, even if it is not bijective, as well as continuous parametric methods which do not sample-split.
	
	On the other hand, the nonparametric calibration methods of isotonic regression \citep{zadrozny2002transforming} and histogram binning \citep{zadrozny2001obtaining} do provide a countable partition of $\Xcal$, and thus may satisfy distribution-free approximate calibration guarantees. 
	In Section~\ref{sec:dfc_guarantees}, we analyze histogram binning and show that any scoring function can be `binned' to achieve distribution-free calibration. We explicitly quantify the finite-sample  approximate calibration guarantees that automatically also lead to asymptotic calibration. We also discuss calibration in the online setting and calibration under covariate shift. 
	\fi 
	\section{Achieving distribution-free calibration}
	\label{sec:dfc_guarantees}
	In Section~\ref{subsec:guarantee-fixed}, we prove a distribution-free approximate calibration guarantee given a fixed partitioning of the feature space into finitely many sets. This calibration guarantee also leads to distribution-free asymptotic calibration. In Section~\ref{sec:data_dep_partition}, we discuss a natural method for obtaining such a partition using sample-splitting, called histogram binning. Histogram binning inherits the bound in Section~\ref{subsec:guarantee-fixed}. This shows that binning schemes lead to distribution-free approximate calibration. In Section~\ref{subsec:online_binning} and \ref{sec:cov_shift} we discuss extensions of this scheme for streaming data and covariate shift respectively. 
	\subsection{Distribution-free calibration given a fixed sample-space partition}
	\label{subsec:guarantee-fixed}
	Suppose we have a fixed partition of $\Xcal$ into $B$ regions $\curlybrack{\Xcal_b}_{b\in [B]}$, and let $\pi_b = \Exp{}{Y \mid X \in \Xcal_b}$ be the expected label probability in region $\Xcal_b$. Denote the partition-identity function as $\Bcal: \Xcal \to [B]$ where $\Bcal(x) = b$ if and only if $x \in \Xcal_b$. 
	Given a calibration set 
	$\{(X_i,Y_i)\}_{i \in [n]}$, let $\nbin_b:= |\{i \in [n]: \Bcal(X_i) = b\}|$ 
	be the number of points from the calibration set that belong to region $\Xcal_b$. In this subsection, we assume that $\nbin_b \geq 1$ (in Section~\ref{sec:data_dep_partition} we show that the partition can be constructed to ensure that $\nbin_b$ is $\Omega(n/B)$ with high probability). Define
	\begin{equation}\label{eq:part_prob_est}
	\small
	\widehat{\pi}_{b} := \frac{1}{\nbin_b}\sum_{i: \Bcal(X_i) =b} Y_i \qquad \text{and} \qquad \widehat{V}_{b} := \frac{1}{\nbin_b} \sum_{i:\Bcal(X_i) = b} (Y_i - \widehat{\pi}_b)^2
	\end{equation}
	as the empirical average and variance of the $Y$ values in a partition. We now deploy an empirical Bernstein bound~\cite{audibert2007tuning} to produce a confidence interval for $\pi_b$. 
	\begin{theorem}\label{thm:emp_bernstein}
		For any $\alpha\in (0,1)$, with probability at least $1-\alpha$,
		\begin{equation*}\label{eq:emp_bernstein_bound}
		\abs{\pi_{b}-\widehat{\pi}_b}\leq \sqrt{\frac{2\widehat{V}_{b}\ln(3B/\alpha)}{\nbin_b}} + \frac{3\ln(3B/\alpha)}{\nbin_b}, \quad \text{simultaneously for all $b\in [B]$}.
		\end{equation*}
	\end{theorem}
	The theorem is proved in Appendix~\ref{appsec:dfc_guarantees}. Using the crude deterministic bound $\widehat{V}_{b}\leq 1$ we get that the width of the confidence interval for partition $b$ is $O(1/\sqrt{\nbin_b})$. However, if for some $b$, $\Xcal_b$ is highly informative or homogeneous in the sense that $\pi_b$ is close to $0$ or $1$, we expect $\widehat{V}_b \ll 1$. In this case,
	Theorem~\ref{thm:emp_bernstein} \emph{adapts} and provides an $O(1/\nbin_b)$ width confidence interval for $\pi_b$. Let $b^\star =\argmin_{b\in [B]}\nbin_b$ denote the index of the region with the minimum number of calibration examples. 
	\begin{corollary}\label{cor:calibration-fixed-partition}
		For $\alpha \in (0, 1)$, the function $h_n(\cdot) := \widehat{\pi}_{\Bcal(\cdot)}$ is distribution-free $(\varepsilon, \alpha)$-calibrated with
		\begin{equation*}
		\small
		\varepsilon =\sqrt{\frac{2\widehat{V}_{b^\star}\ln(3B/\alpha)}{\nbin_{b^\star}}} + \frac{3\ln(3B/\alpha)}{\nbin_{b^\star}}. 
		\end{equation*}
		Thus,  $\{h_n\}_{n \in \naturals}$ is distribution-free asymptotically calibrated for any $\alpha$.
	\end{corollary}
	The proof is in Appendix~\ref{appsec:dfc_guarantees}. Thus, any finite partition of $\Xcal$ leads to asymptotic calibration. However, the finite sample guarantee of Corollary~\ref{cor:calibration-fixed-partition} can be unsatisfactory if the sample-space partition is chosen poorly, since it might lead to small $\nbin_{b^\star}$. In Section~\ref{sec:data_dep_partition}, we present a data-dependent partitioning scheme that provably guarantees that $\nbin_{b^\star}$ scales as $\Omega(n/B)$ with high probability. The calibration guarantee of Corollary~\ref{cor:calibration-fixed-partition} can also be stated conditional on a given test point:
	\begin{equation}
	\abs{\Exp{}{Y \mid f(X)} - f(X)} \leq \varepsilon, \ \text{almost surely $P_X$ }.\label{eq:conditional-calibration}
	\end{equation}
	This holds since Theorem~\ref{thm:emp_bernstein} provides simultaneously valid CIs for all regions $\Xcal_b$.
	\subsection{Identifying a data-dependent partition using sample splitting}
	\label{sec:data_dep_partition}
	Here, we describe ways of constructing the partition $\{\Xcal_b\}_{b\in [B]}$ through histogram binning~\cite{zadrozny2001obtaining}, or simply, binning. Binning uses a sample splitting strategy to learn the partition of $\Xcal$ as described in Section~\ref{subsec:guarantee-fixed}. A split of the data is used to learn the partition and an independent split is used to estimate $\{\widehat{\pi}_b\}_{b\in [B]}$. 
	Formally, the labeled data is split at random into a training set $\Dcal_{\text{tr}}$ and a calibration set $\Dcal_{\text{cal}}$. Then $\Dcal_{\text{tr}}$ is used to train a scoring function $g: \Xcal \to [0, 1]$ (in general the range of $g$ could be any interval of $\mathbb{R}$ but for simplicity we describe it for $[0, 1]$). The scoring function $g$ usually does not satisfy a calibration guarantee out-of-the-box but can be calibrated using binning. 
	
	A \emph{binning scheme} $\mathcal{B}$ 
	is any partition of  $[0,1]$ into $B$ non-overlapping intervals $I_1,\dots,I_B$, such that  $\bigcup_{b\in [B]} I_b = [0,1]$ and $ I_b \cap I_{b'} = \emptyset$ for $b\neq b'$. 
	$\mathcal{B}$ 
	and $g$ induce a partition of $\Xcal$ as follows: 
	\begin{equation}
	\Xcal_b = \curlybrack{x\in\Xcal: \ g(x)\in I_b}, \ b \in  [B].\label{eq:binning-scheme-to-partition}
	\end{equation}
	The simplest binning scheme corresponds to \emph{fixed-width binning}. In this case, bins have the form
	\begin{equation*}
	\small
	I_i = \left[ \frac{i-1}{B}, \frac{i}{B} \right), i=1, \dots, B-1 \ \text{ and } \ I_B = \left[ \frac{B-1}{B}, 1 \right].
	\end{equation*}
	However, fixed-width binning suffers from the drawback that there may exist bins with very few calibration points (low $\nbin_b$), while other bins may get many calibration points. For bins with low $\nbin_b$, the $\widehat{\pi}_b$ estimates cannot be guaranteed to be well calibrated, since the bound of Theorem~\ref{thm:emp_bernstein} could be large. To remedy this, we consider \emph{uniform-mass binning}, which aims to guarantee that each region  
	$\Xcal_b$ contains approximately equal number of calibration points. This is done by estimating the empirical quantiles of $g(X)$. 
	First, the calibration set  $\Dcal_{\text{cal}}$ is randomly split into two parts, $\Dcal_{\text{cal}}^1$ and $\Dcal_{\text{cal}}^2$. For $j \in [B-1]$, the $(j/B)$-th quantile of $g(X)$ is estimated from $\{g(X_i), i\in \Dcal_{\text{cal}}^1\}$. Let us denote the empirical quantile estimates as $\widehat{q}_j$. Then, the bins are defined as:
	\begin{equation*}
	I_1=\left[0,\widehat{q}_1\right),I_{i} = \left[ \widehat{q}_{i-1}, \widehat{q}_{i} \right], i=2,\dots,B-1 \ \text{ and } \ I_B = \left(\widehat{q}_{B-1}, 1 \right].
	\end{equation*}
	This induces a partition of $\Xcal$ as per \eqref{eq:binning-scheme-to-partition}. Now, only $\Dcal_{\text{cal}}^2$ is used for calibrating the underlying classifier, as per the calibration scheme defined in Section~\ref{subsec:guarantee-fixed}. \citet{kumar2019calibration} showed that uniform-mass binning provably controls the number of calibration samples that fall into each bin (see Appendix~\ref{appsec:app_uniform_mass}). Building on their result and Corollary~\ref{cor:calibration-fixed-partition}, we show the following guarantee.
	\begin{theorem}\label{thm:data_dep_partition}
		Fix $g:\Xcal\to [0,1]$ and $\alpha \in (0, 1)$. There exists a universal constant $c$ such that if $\smash{\abs{\Dcal_{\text{cal}}^1} \geq cB \ln (2B/\alpha)}$, 
		then with probability at least 
		$1 - \alpha$, 
		\[
		\nbin_{b^\star} \geq \abs{\Dcal_{\text{cal}}^2}/2B -  \sqrt{\abs{\Dcal_{\text{cal}}^2} \ln(2B/\alpha)/2}.
		\]
		Thus even if $|\Dcal_{\text{cal}}^1|$ does not grow with $n$, 
		as long as $|\Dcal_{\text{cal}}^2| = \Omega(n)$, uniform-mass binning is distribution-free  $(\widetilde{O}(\sqrt{B\ln(1/\alpha)/n}), \alpha)$-calibrated, and hence distribution-free asymptotically calibrated for any $\alpha$.
	\end{theorem}
	The proof is in Appendix~\ref{appsec:dfc_guarantees}. In words, if we use a small number of points (independent of $n$) for uniform-mass binning, and the rest to estimate bin probabilities, we achieve approximate/asymptotic distribution-free calibration. Note that the probability is conditional on a fixed predictor $g$, and hence also conditional on the training data $\Dcal_\text{tr}$. Since Theorem~\ref{thm:data_dep_partition} uses Corollary~\ref{cor:calibration-fixed-partition}, 
	the calibration guarantee 
	can also be stated conditionally on a fixed test point, akin to equation~\eqref{eq:conditional-calibration}.
	\subsection{Distribution-free calibration in the online setting}
	\label{subsec:online_binning}
	So far, we have considered the batch setting with a fixed calibration set of size $n$. However, often a practitioner might want to query additional calibration data until a desired confidence level is achieved. This is called the \emph{online} or \emph{streaming} setting. In this case, the results of Section~\ref{sec:dfc_guarantees} are no longer valid since the number of calibration samples is unknown a priori and may even be dependent on the data. In order to quantify uncertainty in the online setting, we use \emph{time-uniform} concentration bounds~\cite{howard2018UniformNN,howard2020timeuniform}; these hold simultaneously for all possible values of the calibration set size $n \in \mathbb{N}$. 
	
	Fix a partition of $\Xcal$, $\{\Xcal_b\}_{b \in [B]}$. For some value of $n$, let the calibration data be given as 
	$\Dcal_{\text{cal}}^{(n)}$. We use the superscript notation to emphasize the dependence on the current size of the calibration set. 
	Let $\{(X_i^{b},Y_i^{b})\}_{i\in [\nbin_b^{(n)}]}$ be examples from the calibration set that fall into the partition $\Xcal_b$, where $\smash{\nbin_b^{(n)}:= |\{i \in [n]: \Bcal(X_i) = b\}|}$ is the total number of points that are mapped to $\Xcal_b$. Let  the empirical label average and cumulative (unnormalized) empirical variance be denoted as
	\begin{equation}
	\small
      \widehat{V}_{b}^+ = 1\vee \sum_{i=1}^{\nbin_b^{(n)}}\roundbrack{Y_i^{b}-\overline{Y}_{i-1}^{b}}^2, \mbox{ where }\  \overline{Y}_i^{b} :=\frac{1}{i}\sum_{j=1}^{i}Y_j^{b} \mbox{ for } i \in [\nbin_b^{(n)}].
	\end{equation}
	Note the normalization difference between $\widehat{V}_{b}^+$ and $\widehat{V}^{b}$ used in the batch setting. The following theorem constructs confidence intervals for $\{\pi_{b}\}_{b \in [B]}$ that are valid uniformly for any value of $n$.
	
	\begin{theorem}\label{thm:online_conc}
		For any $\alpha\in (0,1)$, with probability at least $1-\alpha$, 
		\begin{equation}
		\small
		\label{eq:uniform_bound}
		\abs{\pi_{b} - \widehat{\pi}_b} \leq  \frac{7\sqrt{\widehat{V}_{b}^+\ln\roundbrack{1+\ln{\widehat{V}_{b}^+}}}+5.3\ln\roundbrack{\frac{6.3B}{\alpha}}}{\nbin_b^{(n)}}, \quad \text{simultaneously for all $b\in[B]$ and all $n \in \mathbb{N}$. }
		\end{equation}%
	\end{theorem}
	The proof is in Appendix~\ref{appsec:dfc_guarantees}. Due to the crude bound: $\widehat{V}_{b}^+\leq \nbin_b^{(n)}$, we can see that the width of confidence intervals roughly scales as $O(\sqrt{\nicefrac{\ln(1+\ln \nbin_{b}^{(n)})}{\nbin_{b}^{(n)}}})$. In comparison to the batch setting, only a small price is paid for not knowing beforehand how many examples will be used for calibration. 
	\subsection{Calibration under covariate shift}\label{sec:cov_shift}
	Here, we briefly consider the problem of calibration under covariate shift~\cite{shimodaira2000improving}. In this setting, calibration data $\smash{\{(X_i,Y_i)\}_{i\in[n]}\sim P^n}$ is from a `source' distribution $P$, while the test point is from a shifted `target' distribution $\smash{(X_{n+1},Y_{n+1})\sim  \widetilde{P} = \widetilde{P}_X\times P_{Y|X}}$, meaning that the `shift' occurs only in the covariate distribution while $P_{Y|X}$ does not change. We assume the likelihood ratio (LR)  
	\[
	\smash{w : \Xcal \to \Real}; \quad  w(x) :=  \dd \widetilde{P}_{X}(x)/ \dd P_{X}(x)
	\] is well-defined. The following is unambiguous: \emph{if $w$ is arbitrarily ill-behaved and unknown, the covariate shift problem is hopeless, and one should not expect any distribution-free guarantees}. Nevertheless, one can still make nontrivial claims using a `modular' approach towards assumptions:
	\begin{enumerate}[itemsep=0cm]    \item[]\hspace{-0.5cm} Condition (A): $w(x)$ is known exactly and is bounded. 
		\item[]\hspace{-0.5cm} Condition (B): an asymptotically consistent estimator $\widehat w(x)$  for $w(x)$ can be constructed. 
	\end{enumerate}
	We show the following: under Condition (A), a weighted estimator using $w$ delivers approximate and asymptotic distribution-free calibration; under Condition (B), weighting with a plug-in estimator for $w$ continues to deliver asymptotic distribution-free calibration. It is clear that Condition (B) will always require distributional assumptions: asymptotic consistency is nontrivial for ill-behaved $w$. Nevertheless, the above two-step approach makes it clear where the burden of assumptions lie: not with calibration step, but with the $w$ estimation step. Estimation of $w$ is a well studied problem in the covariate-shift literature and there is some understanding of what assumptions are needed to accomplish it, but there has been less work on recognizing the resulting implications for calibration. Luckily, many practical methods exist for estimating $w$ given unlabeled samples from $\smash{\widetilde{P}_X}$~\citep{bickel2007discriminative, huang2007, kanamori2009ls}. In summary, if Condition (B) is possible, then distribution-free calibration is realizable, and if Condition (B) is not met (even with infinite samples), then it implies that $w$ is probably very ill-behaved, and so distribution-free calibration is also likely to be impossible. 
	
	For a fixed partition $\smash{\{\Xcal_b\}_{b\in [B]}}$, one can use the labeled data from the source distribution to estimate 
	$\Exp{\widetilde{P}}{Y \mid X\in \Xcal_b}$ (unlike $\Exp{P}{Y \mid X\in \Xcal_b}$ as before), given oracle access  to $w$:  
	\begin{equation}\label{eq:cov_shift_estimator}
	\small
	\widecheck{\pi}^{(w)}_b := \frac{\sum_{i:\Bcal(X_i)=b} w(X_i)Y_i}{\sum_{i: \Bcal(X_i)=b} w(X_i)}.
	\end{equation}
	As preluded to earlier, assume that
	\begin{equation}
	\text{for all }x \in \Xcal,\  L \leq w(x) \leq U \text{ for some } 0 < L \leq 1 \leq U < \infty. \label{eq:cov-shift-assumption}
	\end{equation}
	The `standard' i.i.d. assumption on the test point equivalently assumes $w$ is known and $\smash{L=U=1}$.
	We now present our first claim: $\widecheck{\pi}_b^{(w)}$ satisfies a distribution-free approximate calibration guarantee. To show the result, we assume that the sample-space partition was constructed via uniform-mass binning (on the source domain) with sufficiently 
	many points, as required by Theorem~\ref{thm:data_dep_partition}. This guarantees that all regions 
	satisfy $\abs{\{i : \Bcal(X_i) = b\}} = \Omega(n/B)$ with high probability.
	
	\begin{theorem}\label{thm:weighted_cov_shift_estimator}
		Assume $w$ is known and bounded \eqref{eq:cov-shift-assumption}. Then for an explicit universal constant $\smash{c > 0}$, with probability at least $\smash{1-\alpha}$,
		\begin{equation*}
		\abs{\widecheck{\pi}^{(w)}_{b} - \Exp{\widetilde{P}}{Y \mid X\in \Xcal_b}} \leq c\roundbrack{\frac{U}{L}}^2 \sqrt{\frac{B\ln(6B/\alpha)}{2n}}, \quad \text{simultaneously for all $b\in [B]$},
		\end{equation*}
		as long as $\smash{n \geq c(U/L)^2B\ln^2(6B/\alpha)}$. Thus $h_n(\cdot) = \smash{\widecheck{\pi}_{\Bcal(\cdot)}^{(w)}}$ is distribution-free  asymptotically calibrated for any $\alpha$.
	\end{theorem}
	The proof is in Appendix~\ref{appsec:cov_shift}. Theorem~\ref{thm:weighted_cov_shift_estimator} establishes distribution-free calibration under Condition (A).  
	For Condition (B), using $k$ \emph{unlabeled} samples from the source and target domains, assume that we construct an estimator $\widehat{w}_k$ of $w$ that is consistent, meaning
	\begin{equation}\label{eq:consistent}
	\sup_{x\in\Xcal} \abs{\widehat{w}_k(x)-w(x)} \overset{P}{\rightarrow} 0.
	\end{equation}
We now define an estimator $\widecheck{\pi}^{(\widehat{w}_k)}_b$ by plugging in $\widehat w_k$ for $w$ in the right hand side of \eqref{eq:cov_shift_estimator}:
	\begin{equation*}
	\small
	\widecheck{\pi}^{(\widehat{w}_k)}_b := \frac{\sum_{i:\Bcal(X_i)=b} \widehat{w}_k(X_i)Y_i}{\sum_{i: \Bcal(X_i)=b} \widehat{w}_k(X_i)}.
	\end{equation*}
	\begin{proposition}\label{prop:cov_shift_est_likelihood}
		If $\widehat{w}_k$ is consistent \eqref{eq:consistent}, then  $h_n(\cdot) = \widecheck{\pi}^{(\widehat{w}_k)}_{\Bcal(\cdot)}$ is distribution-free  asymptotically calibrated for any $\alpha \in (0, 0.5)$.
	\end{proposition}
	In Appendix~\ref{appsec:cov_shift}, we illustrate through preliminary simulations that $w$ can be estimated using unlabeled data from the target distribution, and consequently approximate calibration can be achieved on the target domain. Recently, \citet{park2020calibrated} also considered calibration under covariate shift through importance weighting, but they do not show validity guarantees in the same sense as Theorem~\ref{thm:weighted_cov_shift_estimator}. For real-valued regression, 
	distribution-free prediction sets under covariate shift were constructed using conformal prediction~\cite{tibshirani2019conformal} under Condition (A), and is thus a precursor to our modular approach.
	\section{Other related work}\label{sec:related_work}
	The problem of assessing the calibration of binary classifiers was first studied in the meteorological and statistics literature
	~\citep{brier1950verification, sanders1963subjective, murphy1967verification, murphy1972scalara, murphy1972scalarb, murphy1973new, dawid1982well, degroot1983comparison, brocker2012estimating, ferro2012bias}; we refer the reader to the review by~\citet{dawid2014probability} for more details. These works resulted in two common ways of measuring calibration: reliability diagrams~\citep{degroot1983comparison} and estimates of the squared expected calibration error (ECE)~\cite{sanders1963subjective}: $\smash{\E(f(X)-\Exp{}{Y \mid f(X)})^2}$. Squared ECE can easily be generalized to multiclass settings and some related notions such as absolute deviation ECE and top-label ECE have also been considered, for instance~\cite{guo2017nn_calibration,naeini2015obtaining}. ECE is typically estimated through binning, 
	which provably leads to underestimation of ECE 
	for calibrators with continuous output~\citep{vaicenavicius2019bayescalibration, kumar2019calibration}. Certain methods have been proposed to estimate ECE without binning~\citep{zhang2020mix, widmann2019calibration}, but they require distributional assumptions for provability.
	
	While these papers have focused on the difficulty of \textit{estimating} calibration error, ours is the first formal impossibility result for \textit{achieving} calibration. 
	In particular, \citet[Theorem 4.1]{kumar2019calibration} show that the scaling-binning procedure 
	achieves calibration error close to the best within a fixed, regular, injective parametric class. However, as discussed in Section~\ref{subsec:partition} (after Theorem~\ref{thm:f-must-be-coarse}), we show that the best predictor in such an injective parametric class 
	is itself not distribution-free calibrated. 
	In summary, our results show not only that (some form of) binning is necessary for distribution-free calibration (Theorem~\ref{thm:f-must-be-coarse}), but also sufficient (Corollary~\ref{cor:calibration-fixed-partition}). 
	
	Apart from classical methods for calibration \citep{Platt99probabilisticoutputs, zadrozny2001obtaining, zadrozny2002transforming, Niculescu2005predicting}, 
	some new methods 
	have been proposed recently, primarily for calibration of deep neural networks~\cite{lakshminarayanan2017simple,guo2017nn_calibration,kumar2018trainable, tran2019calibrating, seo2019learning, kuleshov2018accurate, kendall2017what, wenger2019non, milios2018dirichlet}. These calibration methods perform well in practice but do not have distribution-free guarantees. 
	A calibration framework that generalizes binning and isotonic regression is Venn prediction~\citep{vovk2004self,vovk2005algorithmic, vovk2012venn, vovk2015large, lambrou2015inductive}; we briefly discuss this framework and show some connections to our work 
	in Appendix~\ref{appsec:venn_prediction}. 
	
	Calibration has natural applications in numerous sensitive domains where uncertainty estimation is desirable (healthcare, finance, forecasting). Recently, calibrated classifiers have been used as a part of the pipeline for anomaly detection~\cite{hendrycks2018deep,lee2018training} and label shift estimation~\cite{saerens2002adjusting, alexandari2019adapting, garg2020unified}.
	\section{Conclusion}\label{sec:discussion}
	We analyzed post-hoc uncertainty quantification for binary classification problems from the standpoint of robustness to distributional assumptions. By connecting calibration to confidence intervals and prediction sets, we established that popular parametric `scaling' methods cannot provide informative calibration in the distribution-free setting. In contrast, we showed that a nonparametric `binning' method --- histogram binning --- satisfies approximate and asymptotic calibration guarantees without distributional assumptions. We also established guarantees for the cases of streaming data and covariate shift. 
	
	\textbf{Takeaway message.} Recent calibration methods that perform binning on top of parametric methods (Platt-binning \citep{kumar2019calibration} and IROvA-TS \citep{zhang2020mix}) have achieved strong empirical performance. In light of the theoretical findings in our paper, we recommend some form of binning as the last step of calibrated prediction due to the robust distribution-free guarantees provided by Theorem~\ref{thm:emp_bernstein}.
	\section{Broader Impact}
	Machine learning is regularly deployed in real-world settings, including areas having high impact on individual lives such as granting of loans, pricing of insurance and diagnosis of medical conditions. Often, instead of hard $0/1$ classifications, these systems are required to produce soft probabilistic predictions, for example of the probability that a startup may go bankrupt in the next few years (in order to determine whether to give it a loan) or the probability that a person will recover from a disease (in order to price an insurance product). Unfortunately, even though classifiers produce numbers between 0 and 1, these are well known to not be `calibrated' and hence not be interpreted as probabilities in any real sense, and using them in lieu of probabilities can be both misleading (to the bank granting the loan) and unfair (to the individual at the receiving end of the decision). 
	
	Thus, following early research in meteorology and statistics, in the last couple of decades the ML community has embraced the formal goal of calibration as a way to quantify uncertainty as well as to interpret classifier outputs. However, there exist other alternatives to quantify uncertainty, such as confidence intervals for the regression function and prediction sets for the binary label. There is not much guidance on which of these should be employed in practice, and what the relationship between them is, if any. Further, while there are many post-hoc calibration techniques, it is unclear which of these require distributional assumptions to work and which do not---this is critical because making distributional assumptions (for convenience) on financial or medical data is highly suspect. 
	
	This paper explicitly relates the three aforementioned notions of uncertainty quantification without making distributional assumptions, describes what is possible and what is not. Importantly, by providing distribution-free guarantees on well-known variants of binning, we identify a conceptually simple and theoretically rigorous way to ensure calibration in high-risk real-world settings. Our tools are thus likely to lead to fairer systems, better estimates of risks of high-stakes decisions, and more human-interpretable outputs of classifiers that apply out-of-the-box in many real-world settings because of the assumption-free guarantees.

	
	\section*{Acknowledgements}
	The authors would like to thank Anurag Sahay, Tudor Manole, Charvi Rastogi, Michael Cooper Stanley, and the anonymous Neurips 2020 reviewers for comments on an initial version of this paper.
	
	\bibliographystyle{plainnat}
	\bibliography{references}
	
	\newpage
	\appendix
	\part*{Appendix}
	The Appendix contains proofs of results in the main paper ordered as they appear. Auxiliary results needed for some of the proofs are stated in Appendix~\ref{appsec:auxiliary}.

	\section{Proof of Proposition~\ref{prop:calibration-characterization}}
	\label{appsec:calib_ci_ps}
	The `if' part of the theorem is due to \citet[Proposition 1]{vaicenavicius2019bayescalibration}; we reproduce it for completeness. Let $\sigma(g), \sigma(f)$ be the sub $\sigma$-algebras generated by $g$ and $f$ respectively. By definition of $f$, we know that $f$ is $\sigma(g)$-measurable and, hence, $\sigma(f)\subseteq \sigma(g)$. 
	We now have:
	\begin{align*}
	\Exp{}{Y \mid f(X)} &=  \Exp{}{\Exp{}{Y  \mid g(X)} \mid f(X)} &\text{(by tower rule since $\sigma(f) \subseteq \sigma(g)$)}
	\\ &= \Exp{}{f(X) \mid f(X)} &\text{(by property~\eqref{eq:calibration-characterization})}
	\\ &= f(X).
	\end{align*}
	The `only if' part can be verified for $g = f$. Since $f$ is perfectly calibrated, 
	\begin{align*}
	\Exp{}{Y \mid f(X) = f(x)} = f(x),
	\end{align*}
	almost surely $P_X$.
	
	\qed

	\section{Proofs of results in Section~\ref{sec:equivalence}}\label{appsec:equivalence}
	
	
	\subsection{Proof of Theorem~\ref{thm:calib_ci_equiv}}
	Assume that one is given a predictor $f$ that is $(\varepsilon,\alpha)$-calibrated. Then the assertion follows from the definition of $(\varepsilon,\alpha)$-calibration since:
	\begin{equation*}
	\abs{\Exp{}{Y \mid f(X)}-f(X)}\leq\varepsilon
	\implies \Exp{}{Y \mid f(X)}\in C(f(X)).
	\end{equation*}
	
	Now we show the proof in the other direction. 
	If $m_C$ was injective, $\Exp{}{Y \mid m_C(f(X))} =\Exp{}{Y \mid f(X)}$ and thus if $\smash{\Exp{}{Y \mid f(X)} \in C(f(X))}$ (which happens with probability at least $1 - \alpha$), we would have $\Exp{}{Y \mid m_C(f(X))} \in C(f(X))$ and so \[\abs{\Exp{}{Y \mid m_C(f(X))} - m_C(f(X)} \leq \sup_{z\in\mathrm{Range}(f)}\curlybrack{|C(z)|/2}  = \varepsilon.\]
	This serves as an intuition for the proof in the general case, when $m_C$ need not be injective.
	Note that,
	{\small
		\begin{align}
		\abs{ \Exp{}{Y \mid m_{C}(f(X))} - m_{C}(f(X))} & = \abs{ \Exp{}{Y \mid m_{C}(f(X))} - \Exp{}{m_{C}(f(X)) \mid m_{C}(f(X))}} \nonumber \\
		& \overset{(1)}{=} \abs{ \Exp{}{\Exp{}{Y \mid f(X)} \mid m_{C}(f(X))} - \Exp{}{m_{C}(f(X)) \mid m_{C}(f(X))}}  \nonumber \\
		& \overset{(2)}{=} \abs{ \Exp{}{\Exp{}{Y \mid f(X)} - m_{C}(f(X)) \mid m_{C}(f(X))} } \nonumber \\
		& \overset{(3)}{\leq}   \Exp{}{\abs{\Exp{}{Y \mid f(X)} - m_{C}(f(X))} \mid m_{C}(f(X))},\label{eq:final-inequality}
		\end{align}}%
	where we use the tower rule in (1) (since $m_{C}$ is a function of $f$), linearity of expectation in (2) and Jensen's inequality in (3). To be clear, the outermost expectation above is over $f(X)$ (conditioned on $m_C(f(X))$). 
	Consider the event \[A: \Exp{}{Y \mid f(X)} \in C(f(X)).\] On $A$, by definition we have:
	\begin{equation*}
	\abs{\Exp{}{Y \mid f(X)} - m_{C}(f(X))}= \frac{u_{C}(f(X))-l_{C}(f(X))}{2} \leq \sup_{z\in  \text{Range}(f)}\roundbrack{\frac{|C(z)|}{2}}= \varepsilon.
	\end{equation*}
	By monotonicity property of conditional expectation, we also have that conditioned on $A$,
	\begin{align*}
	\Exp{}{\abs{\Exp{}{Y \mid f(X)} - m_{C}(f(X))} \mid m_{C}(f(X))} \leq \Exp{}{\varepsilon \mid m_{C}(f(X))} = \varepsilon,
	\end{align*}
	with probability 1. Thus by the relationship proved in the series of equations ending in~\eqref{eq:final-inequality}, we have that conditioned on $A$, with probability $1$,
	\begin{equation*}
	\abs{ \Exp{}{Y \mid m_{C}(f(X))} - m_{C}(f(X))} \leq 
	\varepsilon.
	\end{equation*}
	Since we are given that $C$ is a $(1-\alpha)$-CI with respect to $f$,  $\Prob(A) \geq 1 - \alpha$.
	For any event $B$, it holds that $\Prob\roundbrack{B}\geq \Prob\roundbrack{B|A}\Prob(A)$. Setting \[B : \abs{\Exp{}{Y \mid m_C(f(X))} - m_C(f(X))}   \leq \varepsilon,\] we obtain:
	\begin{align*}
	& \Prob\roundbrack{\abs{\Exp{}{Y \mid m_C(f(X))} - m_C(f(X))}   \leq \varepsilon} \geq 1-\alpha.
	\end{align*}
	Thus, 
	we conclude that $m_C(f(\cdot))$ is $(\varepsilon,\alpha)$-calibrated.
	\qed
	
	\ifx false 
	\subsection{Proof of Corollary~\ref{cor:asymp_calib_CI_diam}}
	Let $\{f_n\}_{n \in \naturals}$ be asymptotically calibrated sequence with the corresponding sequence of functions $\{\varepsilon_n\}_{n \in \naturals}$ that satisfy $\varepsilon_n(f_n(X_{n+1})) = o_P(1)$. 
	From Theorem~\ref{thm:calib_ci_equiv}, we can construct corresponding functions $C_n$ that are $(1-\alpha)$-CI with respect to $f_n$ and satisfy
	\begin{equation*}
	|C_n(f_n(X_{n+1}))| = 2\varepsilon_n(f_n(X_{n+1})) = o_P(1).
	\end{equation*}
	This concludes the proof. 
	\qed
	\fi
	
	\subsection{Proof of Theorem~\ref{thm:PS-CI-equivalence}}
	In the proof, we denote the operation $C(\cdot) \cap \{0,1\}$ as $\text{disc}(C)$ (for `discretize'). Suppose $C_n$ is a $(1-\alpha)$-CI with respect to $f$ for all distributions $P$. We show that $C_n$ covers the label $Y_{n+1}$ itself for distributions $P \in \Pcal_f$ (and thus $\text{disc}(C_n)$ would also cover the labels). 
	
	Consider any distribution $P \in \Pcal_f$ is nonatomic. Fix a set of $\ghost \geq n + 1$ samples from the distribution $P$ denoted as $\Tcal = \{(A^{(j)}, B^{(j)})\}_{j \in [\ghost]}$. Given $\Tcal$, consider a distribution $Q$ corresponding to the following sampling procedure for $(X, Y) \sim Q$:
	\[
	\text{sample an index $j$ uniformly at random from $[\ghost]$} \text{ and set }(X, Y) = (A^{(j)}, B^{(j)}).
	\]
	
	The distribution function for $Q$ is given by 
	\[
	m^{-1}\sum_{j = 1}^m \delta_{(A^{(j)}, B^{(j)})}.
	\]
	
	where $\delta_{(a,b)}$ denotes the points mass at $(a,b)$. Note that $Q$ is only defined conditional on $\Tcal$. Observe the following facts about $Q$:
	\begin{itemize}
		\item \text{supp}($Q) = \{(A^{(j)}, B^{(j)})\}_{j \in [\ghost]}$.
		\item Consider any $(x, y) \in \text{supp}(Q)$. Let $(x, y) = (A^{(j)}, B^{(j)})$ for some $j \in [\ghost]$. 
		Then \begin{align*}
		\Exp{Q }{Y \mid f(X) = f(x)} &= \Exp{Q}{Y \mid f(X) = f(A^{(j)})}
		\\ &\overset{\xi_1}{=}\Exp{Q }{Y \mid X = A^{(j)}} 
		\\ &\overset{\xi_2}{=} B^{(j)} = y.
		\end{align*}
		Above $\xi_1$ holds since $P_{f(X)}$ is nonatomic so that the $f(X^{(i)})$'s are unique almost surely. Note that $P_{f(X)}$ is nonatomic only if $P_X$ itself is nonatomic. Thus the $A^{(j)}$'s are unique almost surely, and $\xi_2$ follow. In other words, if $(X, Y) \sim Q$, then we have
		\begin{equation}
		Y = \Exp{Q}{Y \mid f(X)}.  \label{eq:Q-degenerate}
		\end{equation}
	\end{itemize}
	Suppose the data distribution was $Q$, that is $\{(X_i, Y_i)\}_{i \in [n+1]} \sim Q^{n+1}$. Define the event that the CI guarantee holds as  
	\begin{equation}
	E_1 : \Exp{}{Y_\npo \mid f(X_\npo)} \in  C_n(f(X_\npo)), \label{eq:def-E1}
	\end{equation}
	and the event that the PS guarantee holds as 
	\begin{equation}
	E_2 : Y_{n+1} \in  C_n(f(X_\npo)). \label{eq:def-E2}
	\end{equation}
	Then due to \eqref{eq:Q-degenerate}, the events are exactly the same under $Q$:
	\begin{equation}
	E_1 \overset{Q}{\equiv} E_2.
	\end{equation}
	In particular, this means
	\begin{equation}
	\begin{gathered}
	Q^{n+1} (\Exp{}{Y_\npo \mid f(X_\npo)} \in  C_n(f(X_\npo)))  = Q^{n+1} (Y_\npo \in C_n(f(X_\npo))).
	\end{gathered}
	\label{eq:Q-equivalence}
	\end{equation}
	
	If $C_n$ is a distribution-free CI, then $Q^{n+1}(E_1) \geq 1 - \alpha$ and thus $Q^{n+1}(E_2) \geq 1 - \alpha$. This shows that for $Q$, $\text{disc}(C_n)$ is a $(1-\alpha)$-PI. 
	
	Note that $Q$ corresponds to sampling with replacement from a fixed set $\Tcal$, where each element of $\Tcal$ is drawn with respect to $P$. 
	Although $Q \neq P$, we expect that as $\ghost \to \infty$ (while $n$ is fixed), $Q$ and $P$ coincide. This would prove the result for general $P$. To formalize this intuition, we describe a distribution which is close to $Q$ but corresponds to sampling \emph{without replacement} from $\Tcal$ instead.
	
	For this, now suppose that $\{(X_i, Y_i)\}_{i \in [n+1]} \sim R^{n+1}$ where $R^{n+1}$ corresponds to sampling without replacement from $\Tcal$. Formally, to draw from $R^{n+1}$, we first draw a surjective mapping $\lambda: [n+1] \to [\ghost]$ as
	\[\lambda \sim \text{Unif }(\text{$n$-sized ordered subsets of } [\ghost] ),\]
	and set $(X_i, Y_i) = (A^{(\lambda(i))}, B^{(\lambda(i))})$ for $i \in [n+1]$. 
	
	First we quantify precisely the intuition that as $\ghost \to \infty$, $Q^{n+1}$ and $R^{n+1}$ are essentially identical. Consider the event ``$T : \text{no index is repeated when sampling from $Q^{n+1}$}$". Let $\mathbb{P}(T) = \tau_m$ for some $\ghost$ and note that $\lim_{\ghost \to \infty} \tau_\ghost = 1$.
	Now consider any probability event $E$ over $\{(X_i, Y_i)\}_{i \in [n+1]}$ (such as $E_1$ or $E_2$). We have
	\begin{align*}
	Q^{n+1}(E) &= Q^{n+1}(E | T) \cdot \mathbb{P}(T)   + Q^{n+1}(E | T^c) \cdot \mathbb{P}(T^c)
	\\ &\in [Q^{n+1}(E | T) \cdot \mathbb{P}(T), Q^{n+1}(E | T) \cdot \mathbb{P}(T) +  \mathbb{P}(T^c)]. 
	\end{align*}
	Now observe that $Q^{n+1}(E | T) = R^{n+1}(E)$ to conclude
	\[
	Q^{n+1}(E) \in [R^{n+1}(E) \cdot \mathbb{P}(T), R^{n+1}(E) \cdot \mathbb{P}(T) +  \mathbb{P}(T^c)].
	\]
	Since $m \geq n+1$, $\mathbb{P}(T) \neq 0$ so we can invert the above and substitute $\tau_m = \mathbb{P}(T)$ to get
	\begin{equation}
	R^{n+1}(E) \in \squarebrack{\tau_m^{-1}{(Q^{n+1}(E) - (1-\tau_m))},\ \tau_m^{-1}Q^{n+1}(E)}.    
	\label{eq:R-containment}
	\end{equation}

	
	\noindent Consider $E = E_2$ defined in 
	equation~\eqref{eq:def-E2}. We showed that 
	$Q^{n+1}(E_2)\geq 1 - \alpha$. Thus from \eqref{eq:R-containment},
	\[
	R^{n+1}(E_2) \geq \tau_m^{-1}{(1 - \alpha - (1-\tau_m))}.
	\]
	The above is with respect to $R^{n+1}$ which is  conditional on a fixed draw $\Tcal$. However since the right hand side is independent of $\Tcal$, we can also include the randomness in $\Tcal$ to say:
	\begin{equation}
	\mathbb{P}_{R^{n+1}, \Tcal}(E_2) \geq \tau_m^{-1}(1 - \alpha - (1-\tau_m)). 
	\label{eq:marginalize-over-S}
	\end{equation}
	Observe that if we consider the marginal distribution over $R^{n+1}$ and $\Tcal$ (that is we include the randomness in $\Tcal$ as above), $\{(X_i, Y_i)\}_{i \in [n+1]} \overset{iid}{\sim} P$. (This is not true if we do not marginalize over $\Tcal$, since due to sampling without replacement, the $(X_i, Y_i)$'s are not independent.) Thus equation~\eqref{eq:marginalize-over-S} can be restated as 
	\[
	P^{n+1}(E_2) \geq \tau_m^{-1}(1 - \alpha - (1-\tau_m)),
	\]
	Since $\ghost$ can be set to any number and $\lim_{\ghost \to \infty}\tau_\ghost = 1$, we can indeed conclude
	\[P^{n+1}(E_2) \geq 1 - \alpha.\]
	Recall that $E_2$ is the event that $Y_\npo \in C_n(X_\npo)$; equivalently $Y_\npo \in \text{disc}(C_n(X_\npo))$. Thus $\text{disc}(C_n)$ provides a ($1 - \alpha$)-PI for all $P \in \Pcal_f$. 
	\qed
	
	\subsection{Proof of Corollary~\ref{cor:CI-PS-diameter}}
	
	Consider any distribution $Q \in \Pcal_f$ such that $Q_{f(X)}$ is nonatomic. Then define $P$ such that $P_X = Q_X$ and $\smash{P(Y = 1 \mid X) = 0.5}$ a.s. $Q_X$. Clearly, $P_{f(X)} = Q_{f(X)}$ is nonatomic, so that $P \in \Pcal_f$. Further, $\smash{\Exp{P}{Y_\npo \mid f(X)} = 0.5}$ a.s. $P_{f(X)}$. 
	
	Since $C_n$ is a distribution-free CI w.r.t. $f$ and $P \in \Pcal_f$, by Theorem~\ref{thm:PS-CI-equivalence}, $C_n$ must provide both a prediction set and a confidence interval for $P$:
	\[
	P^{n+1}( \Exp{}{Y_\npo \mid f(X_\npo)} \in C_n(f(X_\npo)) ) \geq 1 - \alpha, 
	\]
	and 
	\[
	P^{n+1}( Y_\npo \in C_n(f(X_\npo)) ) \geq 1 - \alpha. 
	\]
	Thus by a union bound
	\begin{equation} \label{eq:prob-both-in-C}
	P^{n+1}(\{Y_\npo, \Exp{}{Y_\npo \mid f(X_\npo)}\} \subseteq C_n(f(X_\npo))) \geq 1 - 2\alpha.    
	\end{equation}
	Note that if 
	\begin{equation*}
	\{Y_\npo, \Exp{}{Y_\npo \mid f(X_\npo)}\} \subseteq C_n(f(X_\npo)),   
	\end{equation*}
	then  $|C_n(X_\npo)| \geq \abs{Y_\npo - \Exp{}{Y_\npo \mid f(X_\npo)}} \geq 0.5 $. Thus
	\[
	P^{n+1}(|C_n(f(X_\npo))| \geq 0.5) \geq 1 - 2\alpha.
	\]
	Consequently we have 
	\begin{align*}
	\mathbb{E}_{P^{n+1}} |C_n (f(X_\npo))| &\geq 
	0.5(1-2\alpha)
	\\ &= 0.5 - \alpha. 
	\end{align*}
	This concludes the proof. \qed 
	
	\subsection{Proof of Theorem~\ref{thm:f-must-be-coarse}}
	\ifx false 
	Suppose that $\{f_n\}_{n \in \naturals}$ is asymptotically calibrated and satisfies
	\[
	\underset{n \to \infty}{\lim\sup} \abs{\Xcal^{(f_n)}} > \aleph_0,
	\]
	that is, for every $m \in \naturals$, there exists $n \geq m$ such that $\Xcal^{(f_n)}$ is an uncountable set. We will show a contradiction using Corollary~\ref{cor:CI-PS-diameter} for $f_n$ and a certain $C_n$ to be defined shortly. 
	
	First, we verify the condition of Corollary~\ref{cor:CI-PS-diameter} for $f_n$ 
	if $\Xcal^{(f_n)}$ is uncountable: we construct a distribution $P$ such that $P_{(f_n(X))}$ is nonatomic. 
	Let the range of $f_n$ acting on $\Xcal$ be denoted as $f_n(\Xcal)$, and for $z \in f_n(\Xcal)$ let the level set  at value $z$ be denoted as $\Xcal^{(f_n)}_z$.  
	Since the sets $\Xcal^{(f_n)}$ are measurable, we can define $P(X)$ as follows:
	\begin{equation}
	P(f_n(X)) = \text{Unif}(f_n(\Xcal));
	\quad  P(X \mid  f_n(X)) = \text{Unif}\roundbrack{\Xcal^{(f_n)}_{f_n(X)}}.
	\label{eq:example-Pfx}
	\end{equation}
	$P(X)$ along with any conditional probability function $P(Y \mid X)$ constitutes a valid probability distribution $P$. Further, from the construction, since $\Xcal^{(f_n)}$ is uncountable, $P_{f_n(X)}$ is guaranteed to be nonatomic.
	
	Next, since $\{f_n\}_{n \in \naturals}$ is asymptotically calibrated, by Corollary~\ref{cor:asymp_calib_CI_diam}, one can construct a sequence of functions $\{C_n\}_{n \in \naturals}$ such that each $C_n$ is a $(1-\alpha)$-CI with respect to $f_n$ for any distribution $Q$, and 
	\[
	\abs{C_n(f_n(X_{n+1}))}=o_Q(1).  
	\]
	Thus there exists a constant $m$ such that for $n \geq m$ and any distribution $Q$,
	\begin{equation}
	\E_{Q^{n+1}}\abs{C_n(f_n(X_{n+1}))} < 0.5 - \alpha. \label{eq:impossible-condition}    
	\end{equation}
	
	However, since $\underset{n \to \infty}{\lim\sup} |\Xcal^{(f_n)}| > \aleph_0$, there exists an  $n \geq m$ such that $\Xcal^{(f_n)}$ is uncountable. Hence the requirements of Corollary~\ref{cor:CI-PS-diameter} are satisfied by $C_n$ and $f_n$: namely $C_n$ is a $(1-\alpha)$-CI with respect to $f$ for all distributions $P$, and there exists a $P$ such that $P_{f_n(X)}$ is nonatomic. 
	Thus Corollary~\ref{cor:CI-PS-diameter} yields that we can construct a distribution $Q$ such that 
	\[
	\E_{Q^{n+1}}\abs{C_n(f_n(X_{n+1}))} \geq 0.5 - \alpha, 
	\]
	which is a contradiction to \eqref{eq:impossible-condition}. Hence our hypothesis that  $\underset{n \to \infty}{\lim\sup} |\Xcal^{(f_n)}| > \aleph_0$ must be false, concluding the proof. 
	\qed
	\fi 
	Suppose $\Acal$ is distribution-free asymptotically calibrated for some $\alpha \in (0, 0.5)$ and some $\{\varepsilon_n \in [0,1]\}_{n \in \naturals}$ with $\lim_{n \to \infty} \varepsilon_n = 0.$ We show that this assumption leads to a contradiction to Corollary~\ref{cor:CI-PS-diameter}. 
	
	Consider any function $f : \Xcal \to [0,1]$. By the definition of asymptotic calibration, $h_n = \Acal(\Dcal_n, f)$ is $(\varepsilon_n, \alpha)$-calibrated for every $n \in \naturals$. Approximate calibration implies that for the event $E_1: \abs{\Exp{}{Y \mid h_n(X)} - h_n(X)} \leq \varepsilon_n$, we have $P^{n+1}(E_1) \geq 1 - \alpha$.  Following the intuition of Theorem~\ref{thm:calib_ci_equiv}, observe that the event $E_1$ is clearly identical to the event $E_2: \Exp{}{Y \mid h_n(X)} \in [h_n(X) - \varepsilon_n, h_n(X) + \varepsilon_n]$. Thus $P^{n+1}(E_2) \geq 1- \alpha$. Next, note that since the mapping $m_n$ produced by $\Acal$ is injective, $\Exp{}{Y\mid h_n(X)} = \Exp{}{Y\mid m_n(f(X))}=\Exp{}{Y\mid f(X)}$.  Thus, defining $C_n(f(X)) := [m_n(f(X))-\varepsilon_n, m_n(f(X))+\varepsilon_n] = [h_n(X)-\varepsilon_n, h_n(X)+\varepsilon_n]$, we have that 
	\begin{align*}
	    1 - \alpha &\leq P^{n+1}(E_2)
	    \\ &= P^{n+1}(\Exp{}{Y \mid h_n(X)} \in [h_n(X) - \varepsilon_n, h_n(X) + \varepsilon_n])
	    \\ &= P^{n+1}(\Exp{}{ Y \mid f(X)} \in [h_n(X) - \varepsilon_n, h_n(X) + \varepsilon_n])
	    \\ &= P^{n+1}(\Exp{}{Y \mid f(X)} \in C_n(f(X))),
	\end{align*}
	showing that the defined $C_n$ is a distribution-free $(1-\alpha)$-CI w.r.t. $f$. Further, since $\sup_{z \in [0,1]}\abs{C_n(z)} = 2\varepsilon_n$, for any distribution $P$, we have
	\[
	\lim_{n \to \infty}\E_{P^{n+1}}{\abs{C_n(f(X_{n+1}))}} \leq 2\lim_{n \to \infty}\varepsilon_n =0.  
	\]
	Thus, there exists a constant $m$ such that for all $n \geq m$ and any distribution $P$,
	\begin{equation}
	\E_{P^{n+1}}\abs{C_n(f(X_{n+1}))} < 0.5 - \alpha. \label{eq:impossible-condition}    
	\end{equation}
	(Note that this requires $0.5 - \alpha > 0$, which is true since $\alpha \in (0, 0.5)$.) 
	
	Clearly, Corollary~\ref{cor:CI-PS-diameter} is in contradiction to \eqref{eq:impossible-condition}, as long as the assumptions required for Corollary~\ref{cor:CI-PS-diameter} hold. We already have that $C_n$ is a distribution-free $(1-\alpha)$-CI w.r.t. $f$.  All we need to do is exhibit a function $f$ such that $\Pcal_f \neq \emptyset$. Indeed, Lemma~\ref{lemma:p_f_interval} shows that any $f$ whose range contains an interval of $[0,1]$ suffices. 
	
	\ifx false 
	such that $\text{Range}(f) = (0,1)$ is uncountable. We now construct a $Q \in \Pcal_f$. 
	First, define $Q_X = Q_{f(X)} \times Q_{X \mid f(X)}$ as follows:
	\begin{equation}
	Q_{f(X)} = \text{Unif}(f(\Xcal));
	\quad  Q_{X \mid  f(X)} = \text{Unif}\roundbrack{\Xcal^{(f)}_{f(X)}}.
	\label{eq:example-Pfx}
	\end{equation}
	$Q_X$ along with any conditional probability function $Q_{Y \mid X}$ constitutes a probability distribution $Q$ over $\Xcal \times \Ycal$. Since $\Xcal^{(f)}$ is uncountable, $Q_{f(X)}$ is guaranteed to be nonatomic, showing that $Q \in \Pcal_f$, and so $\Pcal_f \neq \emptyset$. \fi 
	

	
	
	Having satisfied the assumptions of Corollary~\ref{cor:CI-PS-diameter}, we conclude that there exists a distribution $Q \in \Pcal_f$ such that 
	\[
	\E_{Q^{n+1}}\abs{C_n(f(X_{n+1}))} \geq 0.5 - \alpha. 
	\]
	This contradicts \eqref{eq:impossible-condition}. Hence our hypothesis that $\Acal$ is distribution-free asymptotically calibrated must be false, concluding the proof. 
	\qed
	
	\subsection{Characterizing a class of functions $f$ for which $\Pcal_f$ is non-empty}
	
	\label{appsec:characterizing-f-pf}

	\begin{lemma}\label{lemma:p_f_interval}
	    If $\text{Range}(f)$ contains a sub-interval of $[0,1]$, then $\Pcal_f$ is non-empty. 
	\end{lemma}
	\begin{proof}
	\ap{
	
    Let the interval $I=[a,b]$ with $a<b\in[0,1]$ be contained in $\text{Range}(f)$, that is,
    \begin{equation}\label{eq:assumption_interval}
        \forall z\in I, \exists x\in \Xcal: f(x)=z.
    \end{equation}
    Let $\lambda$ denote the Lebesgue measure on $[0,1]$ and $\Bcal_{[0,1]}$ the Borel $\sigma$-algebra on $[0,1]$. 
    Define the uniform probability measure $I$ on $P'$:
    \begin{equation}\label{eq:cond_measure}
        P'(S) = \lambda(S\cap I) / \lambda(I);\ \  S \in \Bcal_{[0,1]}.
    \end{equation}
    This is well defined since $\lambda(I) = b  - a > 0$. Clearly, $P'$ does not have atoms on $\Bcal_{[0,1]}$. 
    
    We now want to construct a measure $P^{\star}$ on the Borel $\sigma$-algebra on $\Xcal$ such that the push-forward of $P^{\star}$ under $f$ is $P'$. One can easily check that $\curlybrack{f^{-1}(S) : S\in\Bcal_{[0,1]}}$ defines a $\sigma$-algebra on $\Xcal$. Then, one can define a measure $P^{\star}$ over this $\sigma$-algebra as $P^{\star}(f^{-1}(S)) = P'(S)$. 
    Can $P^{\star}$ be extended to the Borel $\sigma$-algebra over $\Xcal$? \citet{ershov1975extension} studied this problem, leading to the following result. 
    
    \begin{theorem}[Theorem 2.5 by~\citet{ershov1975extension}, adapted]
    Let $\Xcal$ and $\Ycal$ be complete and separable metric spaces with $\Bcal_\Xcal$ and $\Bcal_\Ycal$ being the corresponding Borel $\sigma$-algebras. Let $f:(\Xcal, \Bcal_\Xcal)\to (\Ycal, \Bcal_\Ycal)$ be a measurable mapping and $\nu$ a probability measure on  $(\Ycal, \Bcal_\Ycal)$. 
    If $f$ satisfies
        \begin{equation}\label{eq:ershov_cond}
            f^{-1}(B) \neq \emptyset \quad \text{for all } B\in \Bcal_\Ycal: \nu(B)>0,
        \end{equation}
    then there exists a probability measure $\mu$ on $(\Xcal, \Bcal_\Xcal)$ satisfying
    \begin{equation*}
        \mu(f^{-1}(B)) = \nu (B), \quad \forall B\in \Bcal_\Ycal.
    \end{equation*}
    \end{theorem}

    We invoke Ershov's result with $\Ycal = [0,1]$ and $\nu = P'$. Assumption~\eqref{eq:assumption_interval} guarantees that condition~\eqref{eq:ershov_cond} is fulfilled. We conclude that there exists a probability measure $P^\star$ on $(\Xcal, \Bcal_\Xcal)$, for which $P^\star_{f(X)} = P'$ is non-atomic. Thus $P^{\star} \in \Pcal_f$, concluding the proof.  

}
	\end{proof}
	
	\ifx false 
	\begin{lemma}\label{lemma:p_f_countable}
	    If $\text{Range}(f)$ is finite or countable, then $\Pcal_f = \emptyset$. 
	\end{lemma}
	\begin{proof}
	    Denote $\text{Range}(f) = \{z_1, z_2, \ldots\} \subset [0,1]$. Assume for the sake of contradiction that $\Pcal_f \neq \emptyset$, there exists a distribution over $\text{Range}(f)$ that is nonatomic, say $Q$. Thus $\forall i, Q(z_i) = 0$. Since $\text{Range}(f)$ is at most countable, we have,  $\sum_{z_i\in \text{Range}(f)} Q(z_i) = 0$. However, by the $\sigma$-additivity of measure, $\sum_{z_i\in \text{Range}(f)} Q(z_i) = Q(\text{Range}(f)) = 1$, which is a contradiction.
	\end{proof}
	\fi 
	
	\section{Proofs of results in Section~\ref{sec:dfc_guarantees} (other than Section~\ref{sec:cov_shift})}\label{appsec:dfc_guarantees}

	\subsection{Proof of Theorem~\ref{thm:emp_bernstein}}
    Let $E_{\mathcal{B}(x)}$ be the event that $(\mathcal{B}(X_1),\dots, \mathcal{B}(X_n))=(\mathcal{B}(x_1),\dots,\mathcal{B}(x_n))$. 
	On the event $E_{\mathcal{B}(x)}$, within each region $\Xcal_b$, the number of point from the calibration set is known and the $Y_i$'s in each bin represent independent Bernoulli random variables that share the same mean $\pi_{b}=\Exp{}{Y\mid X\in \Xcal_b}$. Consider any fixed region $\mathcal{X}_{b}$, $b\in [B]$. Using Theorem~\ref{thm:audibert_emp_bernstein}, we obtain that:
	\begin{equation*}
	\Prob\roundbrack{\abs{\pi_{b}-\widehat{\pi}_{b}}> \sqrt{\frac{2\widehat{V}_{b}\ln(3B/\alpha)}{\nbin_{b}}}+\frac{3\ln(3B/\alpha)}{\nbin_{b}} \ \Big| \ E_{\mathcal{B}(x)}}\leq \alpha/B.
	\end{equation*}
	Applying union bound across all regions of the sample-space partition, we get that:
	\begin{equation*}
	\Prob\roundbrack{\forall b\in[B]: \ \abs{\pi_{b}-\widehat{\pi}_b}\leq  \sqrt{\frac{2\widehat{V}_{b}\ln(3B/\alpha)}{\nbin_b}}+\frac{3\ln(3B/\alpha)}{\nbin_b}\ \Big| \ E_{\mathcal{B}(x)}}\geq 1-\alpha.
	\end{equation*}

	Because this is true for any $E_{\mathcal{B}(x)}$, we can marginalize to obtain the assertion of the theorem in unconditional form.
	\qed
	
	\subsection{Proof of Corollary~\ref{cor:calibration-fixed-partition}}
	We convert the per-bin confidence interval of Theorem~\ref{thm:emp_bernstein} to a calibration guarantee using the same intuition as that of Theorem~\ref{thm:calib_ci_equiv}. 
	Define the function $C: [B] \to \interval$ given by 
	\[
	C_n(b) = \squarebrack{\widehat{\pi}_{b} - \roundbrack{\sqrt{\frac{2\widehat{V}_{b}\ln(3B/\alpha)}{\nbin_b}} + \frac{3\ln(3B/\alpha)}{\nbin_b}}, \widehat{\pi}_{b} + \sqrt{\frac{2\widehat{V}_{b}\ln(3B/\alpha)}{\nbin_b}} + \frac{3\ln(3B/\alpha)}{\nbin_b}}, \ b \in [B].
	\]
	Then by Theorem~\ref{thm:emp_bernstein}, $C_n$ provides a `\omaci with respect to $\Bcal : \Xcal \to [B]$'. While Definition~\ref{def:f-confidence-interval} defined CIs with respect to a function whose range is $[0,1]$, it can be naturally extended for CIs with respect to functions with another range such as $\Bcal$. In Section~\ref{subsec:calib_and_ci}, the calibrated function constructed from $C$ was defined as $\widetilde{f}(x) := m_C(f(x))$; the same construction applies even if $\text{Range}(f) \neq [0,1]$. Specifically, for the $C$ defined above, $\widetilde{f}(x) = \widehat{\pi}_{\Bcal(x)}$. The arguments in the proof of the CI-to-calibration part of Theorem~\ref{thm:calib_ci_equiv} give that $\widetilde{f}$ is $(\varepsilon, \alpha)$-calibrated with 
	\[\varepsilon = \sup_{b \in [B]}\abs{C(b)}/2 =
	\sqrt{\frac{2\widehat{V}_{b^\star}\ln(3B/\alpha)}{\nbin_{b^\star}}} + \frac{3\ln(3B/\alpha)}{\nbin_{b^\star}}. 
	\]
	 This shows the approximate calibration result. Next, we show the asymptotic calibration result.   
	
	Suppose some bin $b$ has $\Prob(\Bcal(X) = b) = 0$. Then, a test point $X_\npo$ almost surely does not belong to the bin, and the bin can be ignored for our calibration guarantee. Thus without loss of generality, suppose every $b \in [B]$ satisfies 
	\[
	\Prob(\Bcal(X) = b) > 0. 
	\]
	Let $\min_{b \in [B]} \Prob(\Bcal(X) = b) = \tau > 0$. Then for a fixed number of samples $n$, any particular bin $b$, and any constant $ \alpha \in (0, 1)$ we have by Hoeffding's inequality with probability $1 - \alpha/B$
	\[
	\nbin_b \geq n\tau - \sqrt{\frac{n \ln(B/\alpha)}{2}}.
	\]
	Taking a union bound, we have with probability $1 - \alpha$, simultaneously for every $b \in [B]$, 
	\[
	\nbin_b \geq n\tau - \sqrt{\frac{n \ln(B/\alpha)}{2}} = \Omega(n),
	\]
	and in particular $\nbin_{b^\star} = \Omega(n)$ where $b^\star=\argmin_{b\in [B]}\nbin_b$. Thus by the first part of this corollary, $h_n$ is $\varepsilon_n$ calibrated where $\varepsilon_n = O(\sqrt{n^{-1}}) = o(1)$. This concludes the proof. 
	\qed

	\subsection{Proof of Theorem~\ref{thm:data_dep_partition}}
	Denote $|\Dcal_{cal}^2|=n$. Let $p_j=\Prob(g(X)\in I_j)$ be the true probability that a random point 
	falls into partition $\Xcal_j$. Assume $c$ is such that we can use Lemma~\ref{lem:kumar_uniform} to guarantee that 
	with probability at least $1-\alpha/2$, uniform mass binning scheme is 2-well-balanced. Hence, with probability at least $1-\alpha/2$:
	\begin{equation}\label{eq:lemma_11_1}
	\frac{1}{2B}\leq p_j\leq \frac{2}{B}, \ \forall j \in [B].
	\end{equation}
	Moreover, by Hoeffding's inequality we get that for any fixed region of sample-space partition, with probability at least $1 - \alpha/2B$, for a fixed $j \in [B]$,
	\begin{equation}
	\nbin_j \geq np_j -  \sqrt{\frac{n \ln(2B/\alpha)}{2}}.
	\end{equation}
	Hence, by union bound across applied accross all regions and using~\eqref{eq:lemma_11_1}, we get that with probability at least $1-\alpha/2$:
	\[
	\nbin_{b^\star} \geq \frac{n}{2B} -  \sqrt{\frac{n \ln(2B/\alpha)}{2}},
	\]
	where the first term dominates asymptotically (for fixed $B$). Hence, we get that with probability at least $1-\alpha$, $N_{b^\star}=\Omega\roundbrack{n/B}$. By invoking the result of Corollary~\ref{cor:calibration-fixed-partition} and observing that $\widehat{V}_b \leq 1$, we conclude that uniform mass binning is $(\varepsilon, \alpha)$-calibrated with $\varepsilon = O(\sqrt{B\ln(B/\alpha)/n})$ as desired. This also leads to  asymptotic calibration by Corollary~\ref{cor:calibration-fixed-partition}. 
	\qed
	
	\subsection{Proof of Theorem~\ref{thm:online_conc}}
	
	The proof is based on the result for an empirical-Bernstein confidence sequences for bounded observations~\cite{howard2018UniformNN}. We condition on the event $E_{\mathcal{B}(x)}^\infty$ defined as $(\mathcal{B}(X_1),\mathcal{B}(X_1),\dots)=(\mathcal{B}(x_1),\mathcal{B}(x_2), \dots)$, that is the random variables denoting which partition the infinite stream of samples fall in (thus allowing our bound to hold for every possible value of $n$).  
	On $E_{\mathcal{B}(x)}^\infty$, the label values within each partition of the sample-space partition represent independent Bernoulli random variable that share the same mean $\pi_{b}=\Exp{}{Y\mid X\in\Xcal_b},b\in [B]$. Consequently, the bound obtained can be marginalized over $E_{\mathcal{B}(x)}^\infty$ to obtain the assertion of the theorem in unconditional form. 
	Now we show the bound that applies conditionally on $E_{\mathcal{B}(x)}^\infty$.

	Consider any fixed region of the sample-space partition $\Xcal_{b}$ and corresponding points  $\curlybrack{\roundbrack{X_i^{b},Y_i^{b}}}_{i=1}^{\nbin_{b}}$. Then $S_t = \roundbrack{\sum_{i=1}^t Y_i^{b}} - t\pi_{b}$ is a sub-exponential process with variance process:
	\begin{equation*}
	\widehat{V}_t^+ = \sum_{i=1}^t\roundbrack{Y_i^{b}-\overline{Y}_{i-1}^{b}}^2.   
	\end{equation*}
	\citet[Proposition 2]{howard2020timeuniform} implies that $S_t$ is also a sub-gamma process with variance process $\widehat{V}_t$ and the same scale $c=1$. Since the theorem holds for any sub-exponential uniform boundary, we choose one based on analytical convenience. Recall definition of the polynomial stitching function
	\begin{equation*}
	\mathcal{S}_\alpha (v) := \sqrt{k_1^2vl(v)+k_2^2c^2l^2(v)}+k_2cl(v), \ \text{ where } \begin{cases}
	l(v) := \ln h(\ln_\eta(v/m))+\ln(l_0/\alpha),\\
	k_1 := (\eta^{1/4}+\eta^{-1/4})/\sqrt{2},\\
	k_2 := (\sqrt{\eta}+1)/\sqrt{2}.
	\end{cases}
	\end{equation*}
	where $l_0=1$ for the scalar case. Note that for $c>0$ it holds that $\mathcal{S}_\alpha (v)\leq k_1\sqrt{vl(v)}+2ck_2l(v)$.
	
	From \citet[Theorem 1]{howard2018UniformNN}, it follows that $u(v)=\mathcal{S}_\alpha(v \vee m)$ is a sub-gamma uniform boundary with scale $c$ and crossing probability $\alpha$. Applying Theorem~\ref{thm:howard_uniform} with $h(k) \leftarrow (k+1)^s\zeta(s)$ where $\zeta(\cdot)$ is Riemann zeta function and parameters $\eta\leftarrow e$, $s \leftarrow 1.4$, $c\leftarrow 1$, $m\leftarrow 1$ and $\alpha\leftarrow\alpha/(2B)$, yields that $k_2 \leq 1.88, k_1 \leq 1.46$ and $l(v) = 1.4 \cdot \ln \ln\roundbrack{ev}+\ln (2\zeta(1.4)B/\alpha)$.  Since Theorem~\ref{thm:howard_uniform} provides a bound that holds uniformly across time $t$, then it provides a guarantee for $t=\nbin_{b}$, in particular. Hence, with probability at least $1-\alpha/B$,
	{\small{
			\begin{align*}
			\abs{\pi_{b} - \widehat{\pi}_{b}} &\leq  \frac{1.46 \sqrt{\widehat{V}^+_{b}\cdot 1.4 \cdot \ln \ln\roundbrack{e\roundbrack{\widehat{V}^+_{b}\vee 1}}+\ln (6.3 B/\alpha)}}{\nbin_{b}} 
			+ \frac{5.27 \cdot \ln \ln\roundbrack{e\roundbrack{\widehat{V}^+_{b}\vee 1}}+3.76\ln (6.3B/\alpha)}{\nbin_{b}} \\
			& \leq \frac{7\sqrt{\widehat{V}^+_{b}\cdot \ln \ln\roundbrack{e\roundbrack{\widehat{V}^+_{b}\vee 1}}} + 5.3 \ln (6.3B/\alpha)}{\nbin_{b}}.
			\end{align*}}}
	using that $\sqrt{x+y}\leq\sqrt{x}+\sqrt{y}$ and $\ln\ln(ex)\leq\sqrt{x\ln\ln ex}$ for $x\geq 1$. Finally, we apply a union bound to get a guarantee that holds simultaneously for all regions of the sample-space partition.
	\qed

	\section{Calibration under covariate shift (including proofs of results in Section~\ref{sec:cov_shift})}
	\label{appsec:cov_shift}

	The results from Section~\ref{sec:cov_shift} are proved in Appendix~\ref{appsec:cs-proof-1} (Theorem~\ref{thm:weighted_cov_shift_estimator}) and \ref{appsec:cs-proof-2} (Proposition~\ref{prop:cov_shift_est_likelihood}). To show Theorem~\ref{thm:weighted_cov_shift_estimator}, we first propose and analyze a slightly different estimator than \eqref{eq:cov_shift_est_w_m} that is unbiased for $\pi_b^\bw$, but needs  additional oracle access to the parameters $\{m_b\}_{b \in [B]}$ defined as 
	\[m_b ={P}(X\in \Xcal_b) \ / \ {\widetilde{P}}(X\in\Xcal_b).\] 
	The ratio $m_b$ denotes the `relative mass' of region $\Xcal_b$. (For simplicity, we assume that ${\widetilde{P}}(X \in \Xcal_b) > 0$ for every $b$ since otherwise the test-point almost surely does not belong to $\Xcal_b$ and estimation in that bin is not relevant for a calibration guarantee.) We then show that $m_b$ can be estimated using $w$, which would lead to the proposed estimator $\widecheck{\pi}_b^\bw$. First, we establish the following relationship between 
	$\Exp{\widetilde{P}}{Y \mid X\in \Xcal_b}$ and $\Exp{P}{Y \mid X\in \Xcal_b}$.
	
	\begin{proposition}\label{prop:covariate shift}
		Under the covariate shift assumption, for any $b \in [B]$,
		\begin{equation*}
		\Exp{\widetilde{P}}{Y \mid X\in \Xcal_b } = m_b\cdot  \Exp{P}{w(X)Y\mid X\in \Xcal_b}.
		\end{equation*}
	\end{proposition}
	\begin{proof}
		Observe that 
		\begin{align*}
		\frac{d\widetilde{P}(X \mid X \in \Xcal_b)}{dP(X \mid X \in \Xcal_b)} &= \frac{d\widetilde{P}(X)}{dP(X)}\cdot \frac{P\roundbrack{X\in \Xcal_b}}{\widetilde{P}(X \in \Xcal_b)}
		=  w(X) \cdot m_b.
		\end{align*}
		Thus we have,
		\begin{align*}
		\Exp{\widetilde{P}}{Y \mid X\in \Xcal_b} & \overset{(1)}{=} \Exp{\widetilde{P}}{\Exp{\widetilde{P}}{Y\mid X} \mid X\in \Xcal_b} \\
		& \overset{(2)}{=} \Exp{\widetilde{P}}{\Exp{P}{Y\mid X} \mid X\in \Xcal_b} \\
		& \overset{(3)}{=}  
		\Exp{P}{\frac{d\widetilde{P}(X \mid X \in \Xcal_b)}{dP( X \mid X \in \Xcal_b)}\cdot\Exp{P}{Y\mid X} \mid X\in \Xcal_b} \\
		& \overset{(4)}{=} m_b \cdot \Exp{P}{w(X)\Exp{P}{Y\mid X} \mid X\in \Xcal_b} \\
		& \overset{(5)}{=} m_b \cdot \Exp{P}{\Exp{P}{w(X)Y\mid X} \mid X\in \Xcal_b} \\
		& \overset{(6)}{=} m_b \cdot \Exp{P}{w(X)Y \mid X\in \Xcal_b},
		\end{align*}
		where in (1) we use the tower rule, in (2) we use the covariate shift assumption, (3) can be seen by using the integral form of the expectation, 
		(4) uses the observation at the beginning of the proof, (5) uses that $w(X)$ is a function of $X$ and finally, (6) uses the tower rule.
	\end{proof}
	
	Let $\nbin_b$ denote the number of calibration points from the source domain that belong to bin $b$. Given Proposition~\ref{prop:covariate shift}, a natural estimator for $\Exp{\widetilde{P}}{Y \mid X \in \Xcal_b}$ is given by:
	\begin{equation}
	\label{eq:cov-shift-oracle-mb}
	\widehat{\pi}^{(w)}_{b} :=\frac{1}{\nbin_b}\sum_{i: \Bcal(X_i)=b } m_b w(X_i)Y_i.
	\end{equation}
	
	Estimation properties of $\widehat{\pi}_b^\bw$ are given by the following theorem.
	
	\begin{theorem}\label{thm:cov_shift}
		Assume that $\sup_x w(x)=U<\infty$. For any $\alpha\in(0,1)$, with probability at least $1-\alpha$,
		\begin{equation*}
		\abs{\widehat{\pi}^{(w)}_{b} - \Exp{\widetilde{P}}{Y \mid X\in \Xcal_b}} \leq \sqrt{\tfrac{2\widehat{V}^{(w)}_{b}\ln(3B/\alpha)}{\nbin_b}}+\tfrac{3m_bU\ln(3B/\alpha)}{\nbin_b}, \quad \text{simultaneously for all $b\in [B]$},
		\end{equation*}
		where $ \widehat{V}_{b}^{(w)} = \frac{1}{\nbin_b} \sum_{i:\Bcal(X_i) = b} (m_b w(X_i) Y_i - \widehat{\pi}^{(w)}_b)^2$.
	\end{theorem}
	
	The proof is given in Appendix~\ref{appsubsec:cov_shift}. Next, we discuss a way of estimating $m_b$ using likelihood ratio $w$ instead of relying on oracle access. Observe that 
	\begin{align*}
	\frac{d\widetilde{P}(X \mid X \in \Xcal_b)}{dP(X \mid X \in \Xcal_b)} &= \frac{d\widetilde{P}(X)}{dP(X)}\cdot \frac{P\roundbrack{X\in \Xcal_b}}{{\widetilde{P}}(X \in \Xcal_b)} =  w(X) \cdot m_b.
	\end{align*}
	
	Thus we have, 
	\begin{equation}\label{eq:relationship_m_w}
	\Exp{ P}{w(X)\mid X\in\Xcal_b} = m_b^{-1}\Exp{P}{\frac{d \widetilde{P}(X \mid X \in \Xcal_b)}{d P(X \mid X \in \Xcal_b)}\mid X\in\Xcal_b} = m_b^{-1},
	\end{equation}
	which suggests a possible estimator for $m_b$ given by
	\begin{equation}\label{eq:mb_estimator}
	\widehat{m}_b = \roundbrack{\frac{\sum_{i:\Bcal(X_i)=b}w(X_i)}{\nbin_b}}^{-1}, \quad b\in[B].
	\end{equation}
	On substituting this estimate for $m_b$ in \eqref{eq:cov-shift-oracle-mb}, we get a new estimator 
	\[
	\frac{\sum_{i : \Bcal(X_i) = b}w(X_i)Y_i}{\sum_{i : \Bcal(X_i) = b}w(X_i)},
	\]
	which is exactly $\widecheck{\pi}_b^\bw$. With this observation, we now prove Theorem~\ref{thm:weighted_cov_shift_estimator}. 
	

	\subsection{Proof of Theorem~\ref{thm:weighted_cov_shift_estimator}}
	\label{appsec:cs-proof-1}
	
	Let us define $r_b := 1/m_b$ and 
	\begin{equation}\label{eq:mean_weights}
	\widehat{r}_b = \frac{\sum_{i:\Bcal(X_i)=b}w(X_i)}{\nbin_b}.
	\end{equation}

	\paragraph{Step 1 (Uniform lower bound for $\nbin_b$).} Since the regions of the sample-space partition were constructed using uniform-mass binning, the guarantee of  Theorem~\ref{thm:data_dep_partition} holds. Precisely, we have that with probability at least $1 - \alpha/3$, simultaneously for every $b \in [B]$, 
	\[
	\nbin_{b} \geq \frac{n}{2B} -  \sqrt{\frac{n \ln(6B/\alpha)}{2}}.
	\]
	\paragraph{Step 2 (Approximating $r_b$).} Observe that the estimator~\eqref{eq:mean_weights} is an average of $\nbin_b$ random variables bounded by the interval $[0,U]$. Let $E_{\mathcal{B}(x)}$ be the event that $(\mathcal{B}(X_1),\dots, \mathcal{B}(X_n))=(\mathcal{B}(x_1),\dots,\mathcal{B}(x_n))$. On the event $E_{\mathcal{B}(x)}$, within each region $\Xcal_b$, the number of point from the calibration set is known and the $Y_i$'s in each bin represent independent Bernoulli random variables that share the same mean $\Exp{}{w(X)\mid X\in \Xcal_b}$. Consider any fixed region $\mathcal{X}_{b}$, $b\in [B]$. By Hoeffding's inequality, it holds that
	\begin{equation*}
	\Prob\roundbrack{\abs{r_b-\widehat{r}_b}> \sqrt{\frac{U^2\ln(6B/\alpha)}{2\nbin_{b}}} \ \Big| \ E_{\mathcal{B}(x)}}\leq \alpha/(3B).
	\end{equation*}
	Applying union bound across all regions of the sample-space partition, we get that:
	\begin{equation*}
	\Prob\roundbrack{\exists b\in[B]: \ \abs{r_b-\widehat{r}_b}>\sqrt{\frac{U^2\ln(6B/\alpha)}{2\nbin_{b}}}\ \Big| \ E_{\mathcal{B}(x)}}\leq \alpha/3.
	\end{equation*}
	Because this is true for any $E_{\mathcal{B}(x)}$, we can marginalize to obtain that with probability at least $ 1- \alpha/3$, 
	\begin{equation}\label{eq:imp_weights_approx_error}
	\forall b \in [B], \ \abs{r_b-\widehat{r}_b} \leq \sqrt{\frac{U^2\ln(6B/\alpha)}{2\nbin_{b}}}.
	\end{equation}

	\paragraph{Step 3 (Going from $r_b$ to $m_b$).} Define $\smash{r^\star = \min_{b\in [B]} \Exp{}{w(X)\mid X\in \Xcal_b}}$. Suppose $\forall b \in [B]$, $\abs{r_b - \widehat{r}_b} \leq \varepsilon$ and $\varepsilon<r^\star/2$. Then, we have with probability at least $ 1- \alpha/3$: 
	\begin{equation}
	\abs{m_b - \widehat{m}_b} = \abs{\frac{1}{r_b}-\frac{1}{\widehat{r}_b}}= \abs{\frac{r_b-\widehat{r}_b}{r_b\cdot \widehat{r}_b}} \leq \frac{\varepsilon}{r_b^2|1-\varepsilon/r_b|}\leq \frac{2\varepsilon}{r_b^2}=2 m_b^2\varepsilon, \quad \forall b\in [B].
	\label{eq:connecting-imp-errors}
	\end{equation}
	We now set $\varepsilon = \sqrt{\frac{U^2\ln(6B/\alpha)}{2\nbin_{b}}}$ as specified in equation~\eqref{eq:imp_weights_approx_error} and verify that $\varepsilon < r^\star/2$.
	\begin{itemize}
		\item First, from step 1, with probability at least $ 1- \alpha/3$, 
		$\nbin_{b^\star} = \Omega(n/B)$ and thus $\nbin_{b} = \Omega(n/B)$ for every $b \in [B]$.
		\item 
		By the condition in the theorem statement, for every $b \in [B]$, 
		\[
		\varepsilon = \sqrt{\frac{U^2\ln(6B/\alpha)}{2\nbin_{b}}} = O\roundbrack{\sqrt{\frac{U^2B\ln(6B/\alpha)}{n}}} = O\roundbrack{\sqrt{\frac{U^2B\ln(6B/\alpha)}{\roundbrack{\frac{U^2B\ln(6B/\alpha)}{L^2}}}}} = O\roundbrack{L}.
		\]
		Finally recall that $L\leq r^\star$. Thus we can pick $c$ in the theorem statement to be large enough such that $\varepsilon < L/2 \leq r^\star/2$. 
	\end{itemize}
	
	Thus for $\varepsilon = \sqrt{\frac{U^2\ln(6B/\alpha)}{2\nbin_{b}}}$, by a union bound over the event in \eqref{eq:imp_weights_approx_error} and step 1, the conditions for \eqref{eq:connecting-imp-errors} are satisfied with probability at least $ 1 - 2\alpha/3$. Hence we have for some large enough constant $c > 0$, 
	\begin{equation*}
	\abs{m_b-\widehat{m}_b}\leq c m_b^2\cdot \sqrt{\frac{U^2B\ln(6B/\alpha)}{2n}}\leq  c\cdot \frac{U}{L^2} \sqrt{\frac{B\ln(6B/\alpha)}{2n}}.
	\end{equation*}
	The final inequality holds by observing that $m_b\leq 1/L$ which follows from relationship~\eqref{eq:relationship_m_w} and the assumption that $\inf_x w(x)\geq L$.

	\paragraph{Step 4 (Computing the final deviation inequality for $\widecheck{\pi}_b^\bw$).}
	
	Recall the definitions of  the two estimators:
	\[
	\widehat{\pi}^{(w)}_{b} :=\frac{1}{\nbin_b}\sum_{i: \Bcal(X_i)=b } m_b w(X_i)Y_i, 
	\]
	and
	\[
	\widecheck{\pi}^{(w)}_{b} :=\frac{1}{\nbin_b}\sum_{i:\Bcal(X_i)=b} \widehat{m}_b w(X_i)Y_i,
	\]
	which differ by replacing $m_b$ by its estimator $\widehat{m}_b$ defined in~\eqref{eq:mb_estimator}. By triangle inequality,
	\[
	\abs{\widecheck{\pi}_b-\Exp{}{Y\mid X\in\Xcal_b}}\leq \abs{\widecheck{\pi}_b^{(w)}-\widehat{\pi}_b^{(w)}}+\abs{\widehat{\pi}_b^{(w)}-\Exp{}{Y\mid X\in\Xcal_b}}.
	\]
	
	Theorem~\ref{thm:cov_shift} bounds the term $\abs{\widehat{\pi}_b^{(w)}-\Exp{}{Y\mid X\in\Xcal_b}}$ with high probability. In the proof of Theorem~\ref{thm:cov_shift}, we can replace the empirical Bernstein's inequality by Hoeffding's inequality to obtain with probability at least $ 1- \alpha/3$,
	\begin{equation*}
	\abs{\widehat{\pi}^{(w)}_{b}-\Exp{}{Y\mid X\in\Xcal_b}}\leq\sqrt{\frac{U^2\ln(6B/\alpha)}{2\nbin_b}} \leq \roundbrack{\frac{U}{L}}^2\sqrt{\frac{\ln(6B/\alpha)}{2\nbin_b}},
	\end{equation*}
	simultaneously for all $b\in[B]$ (the last inequality follows since $L \leq 1 \leq U$). To bound $\abs{\widehat{\pi}_b^{(w)}-\widecheck{\pi}_b^{(w)}}$, first note that:
	\begin{align*}
	\abs{\widehat{\pi}^{(w)}_{b}-\widecheck{\pi}^{(w)}_{b}} &= \abs{\frac{1}{\nbin_b}\sum_{i:\Bcal(X_i)=b} \roundbrack{\widehat{m}_b - m_b} w(X_i)Y_i} \\
	&\leq U\cdot \abs{\frac{1}{\nbin_b}\sum_{i:\Bcal(X_i)=b} \roundbrack{\widehat{m}_b - m_b}}\\ 
	& =  U\cdot \abs{\widehat{m}_b - m_b}.
	\end{align*}
	
	Then we use the results from steps 1 and 3 to conclude that with probability at least $1-2\alpha/3$, 
	\[
	\abs{\widecheck{\pi}_b^{(w)}-\widehat{\pi}_b^{(w)}}  \leq  c\cdot \roundbrack{\frac{U}{L}}^2 \sqrt{\frac{B\ln(6B/\alpha)}{2n}}, \ \text{ and } \ \nbin_{b} \geq n/B -  \sqrt{\frac{n \ln(6B/\alpha)}{2}}.
	\]
	simultaneously for all $b \in [B]$. Thus by union bound, we get that it holds with probability at least $1-\alpha$,
	\begin{equation*}
	\abs{\widecheck{\pi}_b-\Exp{}{Y\mid X\in\Xcal_b}} \leq c\cdot \roundbrack{\frac{U}{L}}^2 \sqrt{\frac{B\ln(6B/\alpha)}{2n}},
	\end{equation*}
	simultaneously for all $\smash{b\in [B]}$ and large enough absolute constant $\smash{c>0}$. This concludes the proof. 
	\qed

	\subsection{Proof of Theorem~\ref{thm:cov_shift}}\label{appsubsec:cov_shift}

	Consider the event $E_{\Bcal(x)}$ defined as $(\Bcal(X_1), \dots, \Bcal(X_n))$ $=$ $(\Bcal(x_1), \dots, \Bcal(x_n))$. Conditioned on $E_{\mathcal{B}(x)}$, since  $\sup_{x} w(x)\leq U$, 
	we get that $\widehat{\pi}_{b}^{(w)}$ is an average of independent nonnegative random variables $m_bw(X_i)Y_i$  that are bounded by $m_bU$ and share the same mean $m_b\ \Exp{P}{w(X)Y \mid X\in \Xcal_b} = \Exp{\widetilde{P}}{Y \mid X\in \Xcal_b }$ (by Proposition~\ref{prop:covariate shift}).
	Using Theorem~\ref{thm:audibert_emp_bernstein} for a fixed $b \in [B]$, we obtain:
	\begin{equation*}
	\Prob\roundbrack{\abs{\widehat{\pi}^{(w)}_{b}-\Exp{\widetilde{P}}{Y \mid X\in \Xcal_b}}> \sqrt{\frac{2\widehat{V}_{b}\ln(3B/\alpha)}{\nbin_{b}}}+\frac{3m_bU \ln(3B/\alpha)}{\nbin_{b}} \ \Big| \ E_{\mathcal{B}(x)}}\leq \alpha/B.
	\end{equation*}
	Applying a union bound over all $b \in [B]$, we get:
	{ \small \begin{equation*}
		\Prob\roundbrack{\forall b\in[B]: \ \abs{\widehat{\pi}^{(w)}_{b}-\Exp{\widetilde{P}}{Y \mid X\in \Xcal_b}}\leq \sqrt{\frac{2\widehat{V}_{b}\ln(3B/\alpha)}{\nbin_b}}+\frac{3m_bU\ln(3B/\alpha)}{\nbin_b}\ \Big| \ E_{\mathcal{B}(x)}}\geq 1 - \alpha.
		\end{equation*}}%
	Because this is true for any $E_{\mathcal{B}(x)}$, we can marginalize to obtain the assertion of the theorem in unconditional form.
	\qed

	\subsection{Proof of Proposition~\ref{prop:cov_shift_est_likelihood}}
	\label{appsec:cs-proof-2}
	
	Fix any $\alpha\in (0,0.5)$. For any $k\in\naturals$ observe that by triangle inequality,
	\[
	\abs{\widecheck{\pi}_b^{(\widehat{w}_k)}-\Exp{\widetilde{P}}{Y \mid X\in \Xcal_b}} \leq \abs{\widecheck{\pi}_b^{(w)}-\Exp{\widetilde{P}}{Y \mid X\in \Xcal_b}} + \abs{\widecheck{\pi}_b^{(w)}-\widecheck{\pi}_b^{(\widehat{w}_k)}}.
	\]
	Consider any $\varepsilon>0$. Note that by Theorem~\ref{thm:weighted_cov_shift_estimator}, there exists sufficiently large $n$ such that the first term is larger than $\varepsilon/2$ with probability at most $\alpha/2$ simultaneously for all $b\in [B]$. Hence, it suffices to show that there exists a large enough $k$ such that the probability of the second term exceeding $\varepsilon/2$ is at most $\alpha/2$ simultaneously for all $b\in [B]$. While analyzing the second term, we treat $n$ 
	as a constant while leveraging the  consistency of $\widehat{w}_k$ as $k \to \infty$. For simplicity,  denote $\Delta_k = \sup_x\abs{w(x)-\widehat{w}_k(x)}$. Then for any $b\in [B]$:
	\begin{align*}
	\abs{\widecheck{\pi}_b^{(w)}-\widecheck{\pi}_b^{(\widehat{w}_k)}} & = \abs{ \frac{\sum_{i:\Bcal(X_i)=b} w(X_i)Y_i}{\sum_{i: \Bcal(X_i)=b} w(X_i)} -  \frac{\sum_{i:\Bcal(X_i)=b} \widehat{w}_k(X_i)Y_i}{\sum_{i: \Bcal(X_i)=b}\widehat{w}_k(X_i)}} \\
	& \overset{(1)}{\leq}  \abs{ \frac{\sum_{i:\Bcal(X_i)=b} w(X_i)Y_i}{\sum_{i: \Bcal(X_i)=b} w(X_i)} -  \frac{\sum_{i:\Bcal(X_i)=b} \widehat{w}_k(X_i)Y_i}{\sum_{i: \Bcal(X_i)=b}w(X_i)}} \\
	& +\abs{ \frac{\sum_{i:\Bcal(X_i)=b} \widehat{w}_k(X_i)Y_i}{\sum_{i: \Bcal(X_i)=b} w(X_i)} -  \frac{\sum_{i:\Bcal(X_i)=b} \widehat{w}_k(X_i)Y_i}{\sum_{i: \Bcal(X_i)=b}\widehat{w}_k(X_i)}} \\
	& \overset{(2)}{\leq} n\cdot \Delta_k \cdot \abs{ \frac{1}{\sum_{i: \Bcal(X_i)=b} w(X_i)}} \\
	& +\abs{ \frac{1}{\sum_{i: \Bcal(X_i)=b} w(X_i)} -  \frac{1}{\sum_{i: \Bcal(X_i)=b}\widehat{w}_k(X_i)}}\abs{\sum_{i:\Bcal(X_i)=b} \widehat{w}_k(X_i)Y_i} \\
	& \overset{(3)}{\leq} \frac{n}{L}\cdot \Delta_k +  \roundbrack{\frac{n\cdot \Delta_k}{(L-\Delta_k)L}}\cdot \roundbrack{\roundbrack{U+\Delta_k}\cdot n },
	\end{align*}
	where (1) is due to the triangle inequality, (2) is due to the facts that the number of points in any bin is at most $n$ and that absolute difference between $\widehat{w}$ and $w$ is at most $\Delta_k$, (3) combines the aforementioned reasons in (2) and the assumptions: $L\leq \inf_x w(x) \leq \sup_x w(x)\leq U$. Since $\Delta_k\overset{P}{\rightarrow} 0$, clearly there exists a large enough  $k$ such that:  
	\[
	\Prob\roundbrack{\abs{\widecheck{\pi}_b^{(w)}-\widecheck{\pi}_b^{(\widehat{w}_{k})}}\geq \varepsilon/2}\leq \alpha/2.
	\]
	Thus we conclude that $\widecheck{\pi}_{\Bcal(\cdot)}^{(\widehat{w}_k)}$ is asymptotically calibrated at level $\alpha$.
	\qed
	\subsection{Preliminary simulations}
	This section is structured as follows. We first describe the overall procedure for calibration under covariate shift. The finite-sample calibration guarantee of Theorem~\ref{thm:weighted_cov_shift_estimator} holds for oracle $w$ whereas in our experiments we will estimate $w$; to assess the loss in calibration due to this approximation, we introduce some standard techniques used in literature. The preliminary experiments are performed with simulated data which are described after this. Finally, we propose a modified estimator $\widetilde{\pi}_b^{(\widehat{w})}$ of $\Exp{\widetilde{P}}{Y \mid X \in \Xcal_b}$ which appears natural but has poor performance in practice. 
	
	\paragraph{Procedure.} We describe how to construct approximately calibrated predictions practically. This involves approximating the importance weights $w$ and the relatives mass terms $\{m_b\}_{b \in [B]}$. The summarized calibration procedure 
	consists of the following steps:
	\begin{enumerate}
		\item Split the calibration set into two parts and use the first to perform \emph{uniform mass} binning 
		\item Given unlabeled examples from both source and target domain, estimate $\widehat{w}$. The unconstrained Least-Squares Importance Fitting (uLSIF)  procedure~\cite{kanamori2009ls} is used for this. 
		\item Compute for every $b \in [B]$, the estimator as per~\eqref{eq:cov_shift_estimator}, replacing $w$ with $\widehat{w}$:
		\begin{equation}\label{eq:cov_shift_est_w_m}
		\widecheck{\pi}_{b}^{(\widehat{w})} := \frac{ \sum_{i: \mathcal{B}(X_i)=b } \widehat{w}(X_i) Y_i }{\sum_{i: \mathcal{B}(X_i)=b} \widehat{w}(X_i)  }.
		\end{equation}
		\item On a new test point from the target distribution, output the calibrated estimate $\widecheck{\pi}_{\Bcal(X_{n+1})}^{(\widehat{w})}$.
	\end{enumerate}

	\paragraph{Assessment through reliability diagrams and ECE.} Given a test set (from the target distribution) of size $m$: $\{(X'_i,Y'_i)\}_{i\in [m]}$ and a function $g: \Xcal \to [0, 1]$ that outputs approximately calibrated probabilities, we consider the reliability diagram to estimate its calibration properties. A reliability diagram is constructed using splitting the unit interval $[0,1]$ into non-overlapping intervals $\{I_b\}_{b\in [B']}$ for some $B'$ as \begin{equation*}
	I_i = \left[ \frac{i-1}{B'}, \frac{i}{B'} \right), \ i=1,\dots,B'-1 \ \text{ and } \ I_{B'} = \left[ \frac{B'-1}{B'}, 1 \right].
	\end{equation*}
	Let $\Bcal': [0, 1] \to [B']$ denote the binning function that corresponds to this binning. We then compute the following quantities for each bin $b \in [B']$:
	\begin{align*}
	\text{FP} (I_b) = \frac{\sum_{i: \Bcal'(X_i') = b} Y_i'}{\abs{\{i: \Bcal'(X_i') = b\}}} \qquad &\text{(fraction of positives in a bin)},    \\
	\text{MP}(I_b) = \frac{\sum_{i: \Bcal'(X_i') = b} g(X_i')}{\abs{\{i: \Bcal'(X_i') = b\}}} \qquad &\text{(mean predicted probability in a bin)}.
	\end{align*}
	If $g$ is perfectly calibrated, the reliability diagram is diagonal. Define the proportion of points that fall into various bins as:
	\begin{equation*}
	\widehat{p}_b = \frac{\abs{\{i: \Bcal'(X_i') = b\}}}{m}, \quad b\in [B'].
	\end{equation*}
	Then ECE (or $\ell_1$-ECE) is defined as:
	\begin{equation*}
	\text{ECE}(g) = \sum_{b\in [B']} \widehat{p}_b \cdot  \abs{\text{MP}(I_b)-\text{FP} (I_b)}.
	\end{equation*}
	ECE can also be defined in the $\ell_p$ sense and for multiclass problems but we limit our attention to the $\ell_1$-ECE for binary problems. 
	\begin{figure}[ht]
		\begin{subfigure}{.45\textwidth}
			\centering
			\includegraphics[width=0.9\linewidth]{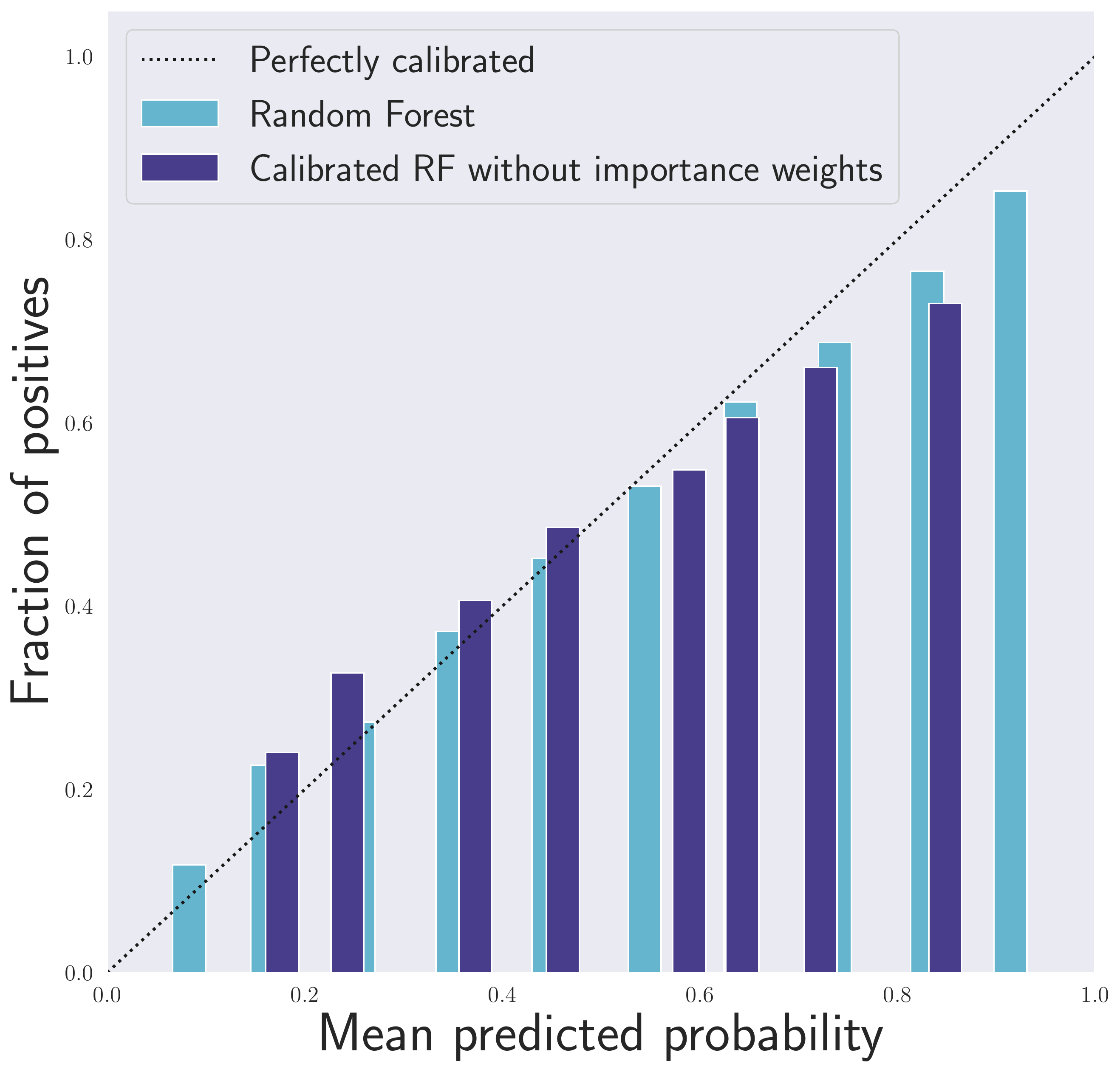}  
			\caption{}
			\label{fig:sub-first}
		\end{subfigure}
		\begin{subfigure}{.45\textwidth}
			\centering
			\includegraphics[width=0.9\linewidth]{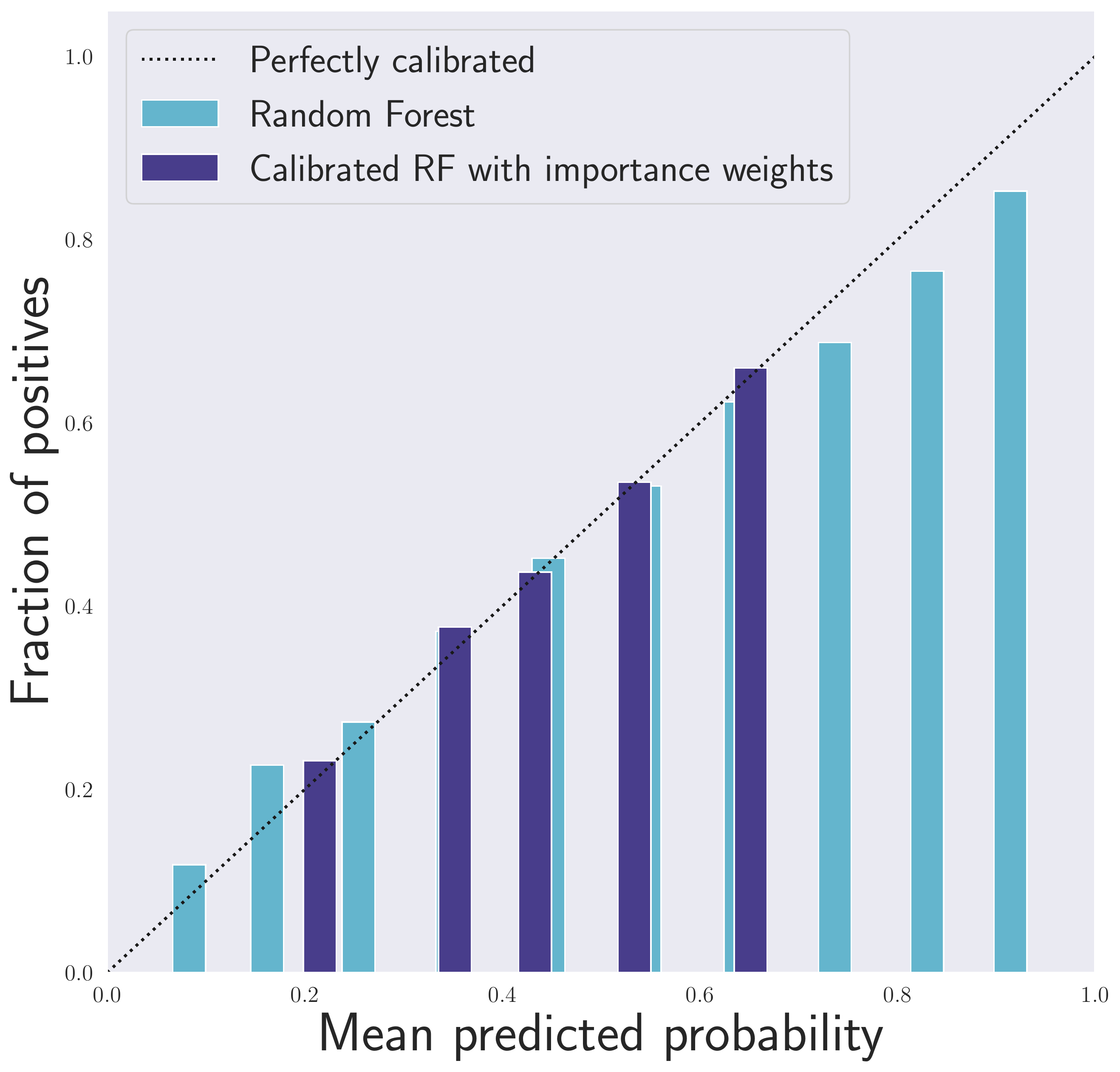}  
			\caption{}
			\label{fig:sub-second}
		\end{subfigure}
		\caption{In Figure~\ref{fig:sub-first} uncalibrated Random Forest ($\text{ECE}\approx 0.023$) is compared with calibration that does not take the covariate shift into account ($\text{ECE}\approx 0.047$). In Figure~\ref{fig:sub-second} uncalibrated Random Forest is compared with calibration that takes the covariate shift into account ($\text{ECE}\approx 0.017$).}
		\label{fig:fig}
	\end{figure}

	\paragraph{Simulations with synthetic data.} 
	We illustrate the performance of our proposed estimator \eqref{eq:cov_shift_estimator} using the following simulated example, for which we can explicitly control the covariate shift. Consider the following data generation pipeline: for the source domain each component of the feature vector is drawn from $\text{Beta}(\alpha,\beta)$ where $\alpha=\beta=1$, which corresponds to uniform draws from the unit cube. For the target distribution each component can be drawn independently from $\text{Beta}(\alpha^\prime,\beta^\prime)$. If the dimension is $d$, the true likelihood ratio is given as 
	\begin{equation*}
	w(x) =  \frac{d\widetilde{P}_X(x)}{dP_X(x)} = \frac{B^{d}(\alpha;\beta)}{B^{d}(\alpha^\prime;\beta^\prime)} \prod_{i=1}^d \frac{(x_{(i)})^{\alpha^\prime-1}(1-x_{(i)})^{\beta^\prime-1}}{(x_{(i)})^{\alpha-1}(1-x_{(i)})^{\beta-1}},
	\end{equation*}
	where $x_{(i)}$ are the coordinates of feature vector $x$. We set $d=3$ and $\alpha^\prime=2,\beta^\prime=1$ so that $w(x) =8\cdot x_{(1)}x_{(2)}x_{(3)}$. The labels for both source and target distributions are assigned according to:
	\begin{equation*}
	\mathbb{P}(Y=1 \mid X=x) = \frac{1}{2}\roundbrack{1+\sin{\roundbrack{\omega\roundbrack{x_{(1)}^2+x_{(2)}^2+x_{(3)}^2}}}},
	\end{equation*}
	for $\omega=20$. As the underlying classifier we use a Random Forest with 100 trees (from \texttt{sklearn}). 14700 data points were used to train the underlying Random Forest classifier, 2000 data points from both source and target were used for the estimation of importance weights. The parameters $\sigma$ and $\lambda$ for uLSIF were tuned by leave-one-out cross-validation: we considered 25 equally spaced values on a log-scale in range $(10^{-2},10^2)$ for $\sigma$ and 100 equally spaced values on a log-scale in range $(10^{-3},10^3)$ for $\lambda$. Uniform mass binning was performed with 10 bins and 1940 data points from the source domain were used to estimate the quantiles. 7840 source data points were used for the calibration and finally, 28000 data points from the target domain were used for evaluation purposes. We note that this simulation is a `proof-of-concept'; the sample sizes we used are not necessarily optimal can presumably be improved. 
	
	We compare the unweighted estimator~\eqref{eq:part_prob_est} 
	which corresponds to weighing points in each bin equally as we would do if there was no covariate shift, and 
	the estimator~\eqref{eq:cov_shift_estimator} that uses an estimate of $w$ to account for covariate shift. The reliability diagrams are presented in Figure~\ref{fig:fig}, with the ECE reported in the caption. For the ECE estimation and reliability diagrams, we used $B' = 10$.
	\begin{figure}[ht]
		\centering
		\includegraphics[width=0.5\linewidth]{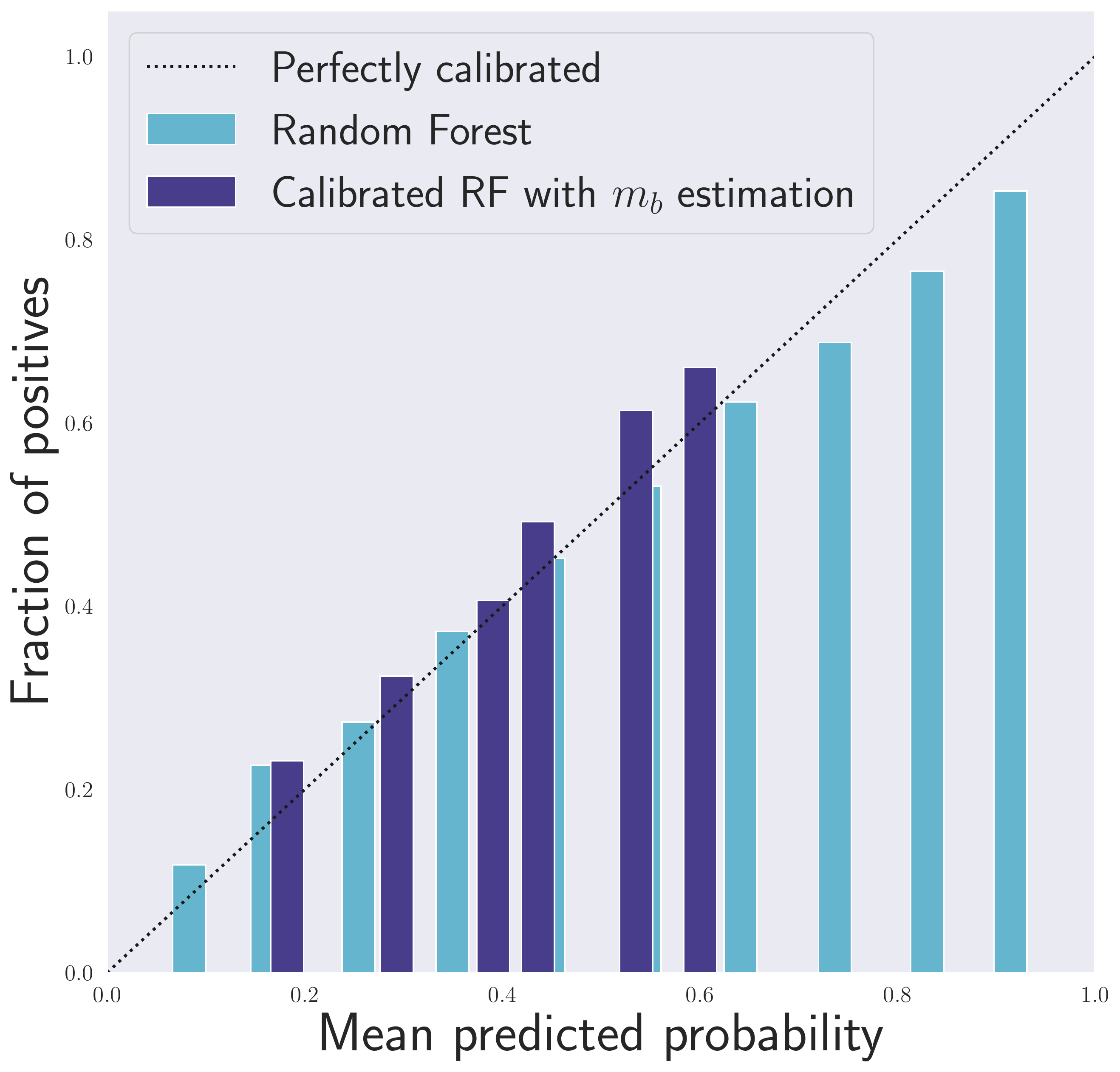}
		\caption{Calibration of Random Forest with $m_b$ estimated as per equation~\eqref{eq:mb_estimator} ($\text{ECE}\approx 0.05$).}
		\label{fig:est_mb}
	\end{figure}
	
	\paragraph{Alternative estimator for $m_b$.} Estimator~\eqref{eq:mb_estimator} is one way of estimating $m_b$ using the $w$ values, that leads to \eqref{eq:cov_shift_estimator}. However, there exists another natural estimator which we propose and show some preliminary empirical results for. 
	Suppose we have access to additional unlabeled data from the source and target domains ($\{X_i^{s}\}_{i\in [n_s]}$, and $\{X_i^{t}\}_{i\in [n_t]}$ respectively). From the definition of $m_b ={P_X}(X \in \Xcal_b)/{\widetilde{P}_X}(X \in \Xcal_b)$, a natural estimator is,
	\begin{equation}\label{eq:m_b_alt_est}
	\widehat{m}_b = \frac{\frac{1}{n_s} \abs{\{i \in [n_s]: \Bcal(X_i^s) = b\}}}{\frac{1}{n_t}\abs{\{i \in [n_t]: \Bcal(X_i^t) = b\}}}, \quad b\in [B].
	\end{equation}
	In this case, the estimator \eqref{eq:cov-shift-oracle-mb} reduces to: 
	\[
	\widetilde{\pi}_b^{(\widehat{w})} = \frac{\widehat{m}_b}{\nbin_b} \sum_{i: \Bcal(X_i) = b} \widehat{w}(X_i) Y_i.
	\]
	We show experimental results with this estimation procedure. We used 8500 data points from the source domain and 8000 points from the target domain to compute~\eqref{eq:m_b_alt_est}. The reliability diagram and ECE with this estimator is reported in Figure~\ref{fig:est_mb}. On our simulated dataset, we observe that the estimators $\widetilde{\pi}_b^{(\widehat{w})}$ perform significantly worse than the estimators $\widecheck{\pi}_b^{(\widehat{w})}$. While this is only a single experimental setup, we outline some drawbacks of this estimation method that may lead to poor performance in general.
	\begin{enumerate}
		\item $\widetilde{\pi}_b^{(\widehat{w})}$ requires access to additional unlabeled data from the source and target domains without leading to increase in performance.
		\item The denominator of $\widehat{m}_b$  could be badly behaved if the number of points from the target domain in bin $b$ are small. We could perform uniform-mass binning on the target domain to avoid this, but in this case $\nbin_b$ may be small which would lead to the estimator $\widetilde{\pi}_b^{(\widehat{w})}$ performing poorly. 
	\end{enumerate}
	Our overall recommendation through these preliminary experiments is to use the estimator $\widehat{\pi}_b^{(\widehat{w})}$ as proposed in Section~\ref{sec:cov_shift} instead of $\widetilde{\pi}_b^{(\widehat{w})}$. 


	\section{Venn prediction}\label{appsec:venn_prediction}
	Venn prediction~\citep{vovk2004self, vovk2005algorithmic, vovk2012venn, lambrou2015inductive} is a  calibration framework that provides distribution-free guarantees, which are different from the ones in Definitions \ref{def:app_calib} and \ref{def:asymp_calib}. 
	For a multiclass problem with $L$ labels, Venn prediction produces $L$ predictions, one of which is guaranteed to be perfectly calibrated (although it is impossible to know which one). These are called multiprobabilistic predictors, formally defined as a collection of predictions $(f_1, f_2, \ldots f_L)$ where each $f_i \in \{\Xcal \to \Delta_{L-1}\}$ (here $\Delta_{L-1}$ is the probability simplex in $\Real^L$). 
	\citet{vovk2012venn} defined two calibration guarantees for multiprobabilistic predictors, the first being oracle calibration.
	\begin{definition}[Oracle calibration]
		$(f_1, f_2, \ldots f_L)$ is oracle calibrated if there exists an oracle selector $S$ such that $f_S$ is perfectly calibrated. 
	\end{definition}
	
	\ifx false 
	\begin{align*}
	\Exp{}{Y|f_Y(X)} = f_Y(X)? \text{ or } \Exp{}{Y|f_S(X),S=Y} = f_S(X)?\\
	\Exp{}{Y|f_1(X),Y=1} = f_1(X), \text{ and } \Exp{}{Y|f_1(X),Y=0} = f_0(X)\\
	\Exp{}{Y|f_1(X)} = f_1(X) \text{ or } \Exp{}{Y|f_0(X)} = f_0(X)
	\end{align*}
	
	\[
	\Exp{}{Y_{n+1} \mid \underbrace{\frac{\sum_{i=1}^{n+1} Y_i}{n+1} = a}_{f_{Y_{n+1}}(X_{n+1})}} = a 
	\]
	Aadi suggestion:
	\[
	\abs{\Exp{}{Y \mid \frac{\sum_{i=1}^n Y_i + \widetilde{Y}}{n+1}} - \frac{\sum_{i=1}^n Y_i + \widetilde{Y}}{n+1}} 
	\]
	\fi 
	
	Venn predictors satisfy oracle calibration \citep[Theorem 1]{vovk2012venn} with $S = Y$. In the binary case, this means that when $Y=1$, $f_1(X)$ is perfectly calibrated but we do not have any guarantee on $f_0(X)$; on the other hand if $Y=0$, $f_0(X)$ is perfectly calibrated but we know nothing about $f_1(X)$. 
	Since $Y$ is unknown, oracle calibration seems to us to  primarily serve as theoretical guidance, but does not give a clear prescription on what to output and what theoretical guarantee that output satisfies. In practice, it seems reasonable to suspect that if $f_0(X)$ and $f_1(X)$ are close, then their average should be approximately calibrated in the sense of Definition~\ref{def:app_calib}, but to the best of our knowledge, such results have not been shown formally (other aggregate functions apart from average are also suggested (without formal guarantees) by \citet[Section 4]{vovk2012venn}).
	For instance, it may be tempting to think that oracle calibration of a multiprobabilistic predictor leads to approximate calibration in the following way. 
	Consider the prediction function 
	\[
	f(X) = \frac{\min f_i(X) + \max f_i(X)}{2},
	\]
	and the radius of the interval $[\min f_i(X), \max f_i(X)]$:
	\[
	\varepsilon(X) =  \frac{\max f_i(X) - \min f_i(X)}{2}.
	\]
	Since Venn predictors satisfy oracle calibration, one might conjecture that $f$ is $(\varepsilon, \alpha)$-calibrated (per Definition~\ref{def:app_calib}) for the given function $\varepsilon$ and for any $\alpha \in (0, 1)$. We examined this claim but were unable to prove such a guarantee formally. In fact, it seems that no general calibration guarantee should
	be possible with the size of the calibration interval being $O(\varepsilon(X))$; we evidence this through the following construction. 
	
	Consider a setup, with no covariates and only label values $Y$, and a single bin that contains all points (in the Venn prediction language: a taxonomy under which all points are equivalent). For a test-point $Y_{n+1}$ and any predictor $f$, note that  $\Exp{}{Y_{n+1} \mid f}$ is simply equal to $\Exp{}{Y_{n+1}}$ since any information used to construct $f$ is independent of $Y_{n+1}$. To ensure calibration, we may look for a guarantee of the following form for some $\delta$: 
	\[
	\abs{\Exp{}{Y_\npo \mid f} - f} = \abs{\Exp{}{Y_\npo} - f} \leq \delta.
	\]
	In essence, $f$ is an estimator  for the parameter $\Exp{}{Y}$ with a corresponding deviation bound of $\delta$. Without distributional assumptions, we only expect to estimate such a parameter with error at best $\delta = O(1/\sqrt{n})$ for a fixed constant probability of failure. On the other hand, the Venn prediction interval $[\min f_i, \max f_i]$ often has radius $O(1/n)$. Thus for valid approximate calibration, we would need to provide a larger interval than $[\min f_i, \max f_i]$, even though one of the $f_i$'s is perfectly calibrated. Given this example, our conjecture is that it might be possible to show that  there always exists an $f_i(X)$ that is $\roundbrack{n^{-0.5}\polylog{1/\alpha}), \alpha}$ calibrated. Without knowing which $f_i(X)$ to pick, perhaps one can show that an aggregate point in the interval $[\min f_i, \max f_i]$ is $((\max f_i - \min f_i) + n^{-0.5}\polylog{1/\alpha}, \alpha)$-calibrated. 
	In Section~\ref{sec:dfc_guarantees}, we showed such a result for histogram binning (which can be interpreted as a Venn predictor). 
	It would be interesting to study if such results can be shown for general Venn predictors.

	Another guarantee for multiprobabilistic predictors is calibration in the large. 
	
	\begin{definition}[Calibration in the large]
		$(f_1, f_2, \ldots f_L)$ is calibrated in the large if the following is satisfied:
		$\Exp{}{Y} \in [\E \min f_i(X), \E \max f_i(X)]$.
	\end{definition}
	\citet[Theorem 2]{vovk2012venn} show that Venn predictors satisfy calibration in the large. Due to the expectation signs and the coverage of the marginal probability $\Exp{}{Y}$, calibration in the large does not lead to a clear interpretable guarantee for uncertainty quantification, but rather a minimum requirement that serves as a guiding principle. 

	\section{Auxiliary results}
	\label{appsec:auxiliary}
	
	\subsection{Concentration inequalities}
	\begin{theorem}[\citet{howard2018UniformNN}, Theorem 4]\label{thm:howard_uniform}
		Suppose $Z_t\in [a,b]$ a.s. for all $t$. Let $(\widehat{Z}_t)$ be any $[a,b]$-valued predictable sequence, and let $u$ be any sub-exponential uniform boundary with crossing probability $\alpha$ for scale $c=b-a$. Then:
		\begin{equation*}
		\Prob\roundbrack{\forall t\geq 1: \abs{\overline{Z}_t-\mu_t}<\frac{u\roundbrack{\sum_{i=1}^t\roundbrack{Z_i-\widehat{Z}_i}^2}}{t}}\geq 1-2\alpha.
		\end{equation*}
	\end{theorem}

	\begin{theorem}[Partial statement of \citet{audibert2007tuning}, Theorem 1]\label{thm:audibert_emp_bernstein}
		Let $X_1,\dots,X_n$ be i.i.d. random variables bounded in $[0,s]$, for some $s > 0$. Let $\mu=\Exp{}{X_1}$ be their common expected value. Consider the empirical mean $\overline{X}_n$ and variance $V_n$ defined respectively by
		\begin{equation*}
		\overline{X}_n=\frac{\sum_{i=1}^n X_i}{n}, \quad \text{ and } \quad V_n=\frac{\sum_{i=1}^n(X_i-\overline{X}_n)^2}{n}.
		\end{equation*}
		Then for any $\delta \in (0, 1)$, with probability at least $1-\delta $,
		\begin{equation*}
		\abs{\overline{X}_n-\mu}\leq \sqrt{\frac{2V_n \log(3/\delta)}{n}}+\frac{3s\log(3/\delta)}{n}.
		\end{equation*}
	\end{theorem}

	\subsection{Uniform-mass binning}\label{appsec:app_uniform_mass}
	
	\citet{kumar2019calibration} defined well-balanced binning and showed that uniform mass-binning is well-balanced. 
	\begin{definition}[Well-balanced binning]
		A binning scheme $\mathcal{B}$ of size $B$ is $\beta$-well-balanced $(\beta\geq 1)$ for some classifier $g$ if
		\begin{equation*}
		\frac{1}{\beta B}\leq \mathbb{P}\left( g(X)\in I_b \right)\leq \frac{\beta}{B}, 
		\end{equation*}
		simultaneously for all $b \in [B]$.
	\end{definition}
	To perform uniform-mass binning labeled examples are required at the stage of training the base classifier $g(\cdot)$. We denote this data as $\Dcal^1_{\text{cal}}$. Procedures based on uniform-mass binning are well-balanced if $\abs{\Dcal_{\text{cal}}^1}$ is sufficiently large.
	\begin{lemma}[\citet{kumar2019calibration}, Lemma 4.3]\label{lem:kumar_uniform}
		For a universal constant $c>0$, if $\abs{\Dcal_{\text{cal}}^1} \geq cB \ln (B/\alpha)$, then with probability at least $1 - \alpha$, the  uniform mass binning scheme $\mathcal{B}$ is 2-well-balanced.
	\end{lemma}
	The calibration guarantees in Section~\ref{sec:dfc_guarantees} depend on the minimum number of training points $\nbin_{b^\star}$ in any bin. Uniform mass-binning guarantees that $\nbin_{b^\star} = \Omega(n/B)$. This is used in the proof of Theorem~\ref{thm:data_dep_partition}.
\end{document}